\documentclass[runningheads]{llncs}
\usepackage{graphicx}

\usepackage{tikz}
\usepackage{style-files/comment} 
\usepackage{amsmath,amssymb} 
\usepackage{color}

\usepackage{graphicx}
\usepackage{comment}
\usepackage{amsmath,amssymb} 
\usepackage{color}
\usepackage{epsfig, graphicx, amsmath, amssymb}
\usepackage{multirow,lineno,comment,subfigure,wrapfig,booktabs}
\usepackage[inline,shortlabels]{enumitem}
\usepackage{epsfig}
\usepackage{graphicx}
\usepackage{amsmath}
\usepackage{amssymb}
\usepackage{url}  			
\usepackage{booktabs}       
\usepackage{amsfonts}       
\usepackage{nicefrac}       
\usepackage{microtype}      
\usepackage{mathtools,amsmath,amssymb,verbatim,tabularx,ragged2e,multirow, algorithm, algorithmic}

\usepackage{amsthm,comment,multicol,enumitem,setspace,wrapfig,xcolor}
\usepackage[font=small,skip=0pt]{caption}
\usepackage{amsmath}
\usepackage{algorithm}
\usepackage{mathbbol}

\usepackage{hyperref}

\DeclareMathOperator*{\argmin}{arg\,min}
\DeclareMathOperator*{\argmax}{arg\,max}

\newcounter{ALC@tempcntr}
\newcommand{\LCOMMENT}[1]{
	\setcounter{ALC@tempcntr}{\arabic{ALC@rem}}
	\setcounter{ALC@rem}{1}
	\item[\#] {#1} 
	\setcounter{ALC@rem}{\arabic{ALC@tempcntr}}
}
\newcommand{\LBLANKCOMMENT}{%
	\setcounter{ALC@tempcntr}{\arabic{ALC@rem}}
	\setcounter{ALC@rem}{1}
	\item {} 
	\setcounter{ALC@rem}{\arabic{ALC@tempcntr}}
}

\def\etal{{et al.~}}
\def\transpose{^\intercal}

\def\l2{\ell_2}
\def\ml2{$\l2$}
\def\cs{{\mathbb{c}}}

\def\VS{{\mathcal{V}}}
\def\CS{{\mathcal{C}}}

\def\DS{{\mathcal{D}}}
\def\GS{{\mathcal{G}}}
\def\AS{{\mathcal{A}}}
\def\BS{{\mathcal{B}}}
\def\HS{{\mathcal{H}}}
\def\ES{{\mathcal{E}}}
\def\IS{{\mathcal{I}}}
\def\PS{{\mathbb{P}}}

\def\real{{\mathbb{R}}}

\def\poset{{\Pi(\DS)}}
\def\partition{{\mathcal{P}}}

\def\Xhat{{\hat{X}}}
\def\Xbar{\overline{X}}
\def\Ybar{\overline{Y}}
\def\Sbar{\overline{S}}
\def\ra{{\rightarrow}}

\def\o{ {\prec} }
\def\oe{ {\preceq} }
\def\ob{ {\overline{\o}} }

\def\ot{ {\widetilde{\o}} }

\def\eb{b}
\def\ebo{\eb^\prec}

\newcommand{\dotprod}[2]{{\langle #1,#2 \rangle}}
\newcommand{\norm}[1]{{\left\lVert#1\right\rVert}}
\newcommand{\card}[1]{{\lvert #1 \rvert}}

\newcommand{\be}{\begin{equation}}
\newcommand{\ee}{\end{equation}}
\newcommand{\bea}{\begin{eqnarray}}
\newcommand{\eea}{\end{eqnarray}}

\newcommand{\ue}[1]{\underline{\emph{#1}}}

\newtheorem*{theorem*}{Theorem}
\newtheorem*{lemma*}{Lemma}
\newtheorem*{definition*}{Definition}
\newtheorem*{corollary*}{Corollary}

\begin{document}
\pagestyle{headings}
\mainmatter
\def\ECCVSubNumber{3414}  

\title{Avoiding Exponential Blow-Up in Distributive Lattice Submodular Minimization}


\titlerunning{Distributive Lattice Submodular Minimization}
%
\author{Ishant Shanu}
\institute{}
\authorrunning{I. Shanu}
%
\maketitle

    \begin{abstract}
	Submodular function minimization has gained a lot of interest in recent years. They are highly applicable in the area of Computer Vision and Machine Learning. Often such applications require to work with submodular functions defined on distributive lattice. Current best way of dealing with it is using a transformation which extrapolates the submodular function for the respective boolean lattice. It makes optimization system too inefficient due to enlargement of the working space. Quantitatively, the expanded space has additional exponential (in set size) number of elements. We propose a generic framework for dealing with distributive lattice which only works within distributive lattice. Our framework allows one to use already established submodular function minimization algorithms for boolean lattice. In our experiment, we show the huge improvement in terms of running time over tranditional methods for handling distributive lattice.   
\end{abstract}

    \section{Introduction}

Let $\DS$ be a subset of $2^{\VS}$, $\VS$ a finite non-empty set, that is closed under under the set union and intersection operations. Not only is $\DS$ a distributive lattice,  Birkhoff \cite{birkhoff1937rings} has shown that any finite distributive lattice is  isomorphic to one defined as a subset of some power set. Distributive lattices are being used to model distributed scheduling \cite{yi2007mac,yi2010mac,ammann1993lattice}, formal concept analysis \cite{gely2022}, computer vision \cite{mlgc}, learning through hypothesis generation \cite{kuznetsov2004learning,babin2017dualization}, and in pattern
mining \cite{nourine2012extending}. Considering the spectrum of applications, it is not surprising that research has focused on issues ranging from representations that allow efficient calculation of meet and join operations, ideal enumeration to minimizing functions defined over elements of distributive lattices. 

Our focus in this paper is on algorithms that are practically efficient for distributive lattice function minimization. We limit ourselves to functions that are submodular, i.e.  functions $F: \DS \rightarrow \real$ defined on the distributive lattice $\DS \subseteq 2^{\VS}$ that satisfy,
\begin{align}
F(X) + F(Y ) \geq F(X \cap Y ) + F(X \cup Y ) \quad \forall (X,Y \in \DS).
\end{align}

It has been known for some time that submodular functions defined both over Boolean lattices ($\DS$ is called a Boolean lattice when $\DS = 2^{\VS}$) and distributive lattices can be minimized in strongly polynomial time \cite{schrijver00,orlin09,iwata01,iwata2008submodular}. However, the time complexity of these algorithms is bounded by very high degree polynomials (at least of degree 5) which limits their practical suitability and usefulness. The algorithms are submodular polyhedron based. They are combinatorial in the sense that the optimization problem that is being solved is modeled as a flow problem. Interestingly, the submodular polyhedron based algorithm that is practically the most efficient is based on  minimum-norm-point algorithm \cite{fujishige2006minimum}. While experimentally the observed time complexity is $O(n^{3.5})$ \cite{mccormick2005submodular} its convergence is only guaranteed in principle. The best theoretical estimates of time complexity are at least $O(n^9)$ \cite{chakrabarty2014provable}. Coming up with a reasonable estimate of time complexity for  minimum-norm-point based algorithms for submodular function minimization is of considerable interest.

There is another line of research on practically efficient algorithms for submodular function minimization that has emerged from computer vision where many problems have been modeled requiring computation of maximum a posteriori in a Markov random field (MAP-MRF) \cite{survey2015}. Typically the energy fields are defined over thousands of clusters of pixels (called cliques) in images involving millions of pixels. Individual energy fields defined over cliques are submodular functions and minimization problem can be looked upon as minimizing sum of submodular (SoS) functions. Depending upon the problem clique sizes can range from 2 to 1000. When cliques are of size 2 (order 2)the simplest SoS function problem (involving only 2 labels per pixel) has a natural max flow interpretation \cite{kolmogorov2004energy} and the dynamic tree based max flow algorithm called Graph Cuts \cite{boykov2001fast} is the algorithm of choice. Optimal solutions to higher order SoS submodular functions arising in computer vision was possible only in principle (none of the strongly polynomial algorithms could scale) till a gadget was devised to model flow in a high order clique and a max flow min cut theorem proved \cite{gc}. Since residual capacity calculation was exponential in the size of a clique the algorithm did not scale beyond cliques of size the  for the resultant flow graph.

to 
of order  is strongly polynomial. n polynomial time. Success of the submodular function minimization algorithms come by solving an equivalent problem defined on the base polyhedron.

Current best way of minimizing a submodular function on distributive lattice is to define the equivalent submodular function on the boolean lattice \cite{schrijver00}. Then use algorithms developed to minimise the derived submodular functions on the boolean lattice. This method increases the size of the domain of submodular function and makes optimization too inefficient. Such technique of transforming the domain space can add as large as an exponential (in $\VS$) number of elements in the domain space. Which makes the algorithm intractable for any practical usage.

Some of the notable submodular function minimization techniques over boolean lattice are by Schrijver \cite{schrijver2003combinatorial}, Iwata \cite{iwata01}, Orlin \cite{orlin09}, MNP \cite{fujishige06}. The mentioned methods optimize an equivalent dual problem. Some of the earlier work directly work on the primal domain. A submodular function can be converted into a piecewise continuous convex function using lovasz extension \cite{bach2013learning}. Subgradient methods \cite{boyd2003subgradient} have been used to minimize the lovasz extension of submodular function. Recent work along the line are by Chakrabarty et.al. \cite{chakrabarty2017subquadratic},  Axelrod et.al. \cite{axelrod2020near}.

We introduce a general framework for minimising the submodular functions defined on distributive lattice. Our framework allows algorithms to run only on the elements of the distributive lattice. Which makes huge improvement in running time as compared with the current state of the art method. We will show how to adapt our framework on the algorithms developed for minimising submodular function over boolean lattice.

At this point, we should bring a closely related work by Shanu et. al. \cite{shanu2020inference} which happens to be a special case of our generic framework. They consider the problem defined on the pixels set of images. Where each pixel can have one of $k$ labels from a predefined label set. In theory the pixel set can be seen as set on which a k-submodular functions is defined. Note that in their work every pixel can have label from a fixed label set. On the other hand, using our framework one can work with problem in which every pixel can have different label set.

Organization of the paper is as follows. In the following section \ref{sec:sfmboolean}, we will briefly look at the technique for submodular function minimization over boolean lattice. In section \ref{sec:DL} we give a detailed analysis of already established distributive lattice system. In section \ref{sec:tranformation}, we introduce the new transformation of submodular function from distributive lattice to boolean lattice. In section \ref{sec:invalidbase}, we show the structure that arises in the base polyhedron of the transformed function. We further show that the base span of base vectors can be divided into two spaces. One which is derived from the elements in the distributive lattice. Second space which is derived from the extra elements introduced in the domain through the transformation. In our main result (Theorem \ref{thm:inv-block-rep}), we show that space derived upon extra states can be represented as the span of linear number of elementary basis vectors. In section \ref{sec:invalidl2norm} we give flow based algorithm to optimize the L2 norm of the vector in span of elementary basis vectors. In section \ref{sec:mainalgo}, we adapt the practically fastest algorithm MNP \cite{fujishige2006minimum} for minimizing submodular function. In follow up sections we give proof of the algorithm's convergence. In the experiment section, we show the comparison of running time of algorithms with and without using proposed framework.

	\section{Background}
\label{sec:background}
In the following section section we briefly describe the basic terminology and results from \emph{submodular function minimization} (SFM) literature required to follow the discussion in this paper. We direct the reader to \cite{schrijver2003combinatorial} for more details. 

\subsection{SFM over boolean lattice}
\label{sec:sfmboolean}
The objective of a SFM problem is to find a minimizer set, $S^* = \min_{S \subseteq \VS} f(S)$ of a submodular function $f$, where $\VS$ is the set of all the elements. W.l.o.g. we assume $f(\phi)=0$. We associate two polyhedra in $\real^{|\VS|}$ with $f$, the \emph{submodular polyhedron}, $P(f)$, and the \emph{base polyhedron}, $B(f)$, such that
\begin{align*}
P(f) = & \{x \mid x \in \real^{|\VS|}, ~ \forall ~ U \subseteq \VS: x(U) \leq f(U) \}, \; \text{and} \\
B(f) = & \{x \mid x \in P(f), x(\VS) = f(\VS) \}, 
\end{align*}
where $x(v)$ denotes the element at index $v$ in the vector $x$, and  $x(U) = \sum_{v \in U} x(v)$. A vector in the base polyhedron $B(f)$ is called a \emph{base}, and an extreme point of $B(f)$ is called an \emph{extreme base}. \emph{Edmond's greedy algorithm} gives a procedure to create an extreme base, $b^\o$, given a total order $\o$ of elements of $\VS$ such that $\o: v_1 \o \ldots \o v_n$, where $n=\card{\VS}$. Denoting the first $k$ elements in the ordered set $\{v_1, \ldots,v_k, \ldots, v_n\}$ by $k_\o$, the algorithm initializes the first element as $b^\o(1) = f(\{v_1\})$ and rest of the elements as $b^\o(k) = f(k_\o) - f((k-1)_\o)$. There is a one to one mapping between an ordering of the elements, and an extreme base. The \textit{Min Max Theorem}, states that $ \max \{ x^-(\VS) \mid x \in B(f) \} = \min \{ f(U) \mid U \subseteq \VS \}$. Here, $x^-(\VS)$ gives the sum of negative elements of $x$. 

The min-norm equivalence result shows that $\argmax_{x \in B(f)} x^-(\VS) = \argmin_{x \in B(f)} \\ \norm{x}_2$. Fujishige and Isotani's \cite{fujishige11}  \emph{Min Norm Point} (MNP) algorithm uses the equivalence and solves the problem using Wolfe's algorithm \cite{fujishige2006minimum}. The algorithm has been shown empirically to be the fastest among all base polyhedron based algorithms \cite{jegelka13,sosmnp}.  The algorithm maintains a set of extreme bases, $\{b^{\o_i} \}$, and a minimum norm base vector, $x$, in their convex hull, s.t.:
\begin{equation}
\label{eqn:convexcombination}
x=\sum_i \lambda_i b^{\o_i} \quad  \lambda_i \geq 0, \text{ and }  \sum_i \lambda_i=1. 
\end{equation}
At a high level, an iteration in the MNP/Wolfe's algorithm comprises of two stages. In the first stage, given the current base vector, $x$, an extreme base, $q$, that minimizes $x^{\intercal}q$ is added to the current set. The algorithm terminates in case $\norm{x}=x\transpose q$. Otherwise it finds a new $x$, with smaller norm, in the convex hull of the updated set of extreme bases. 

Minimizing a sum of submodular functions \cite{kolmogorov12,sosmnp} has also gain attention and have practical applications like finding MRF-MAP inference. Shanu \etal  \cite{sosmnp}  have suggested a block coordinate descent framework  to  implement  the  Min  Norm  Point  algorithm in the sum of submodular functions environment for solving MRF-MAP inference problem. A very broad overview of that scheme is as follows.

The submodular function $f$ defined as sum $f(S)= \sum_{\cs \in \CS} f(\cs \cap S) \forall S \subseteq \VS$ of $f_{\cs}$'s. With each submodular function $f_\cs$, one can associate a base polyhedron such that:
\begin{align}
B(f_\cs) := \Big\{y_\cs \in \real^{|\cs|} \mid ~ y_\cs(U) \leq f_\cs(U), ~ \forall U \subseteq \cs ~ ; ~ 
y_\cs(\cs) = f_\cs(\cs) \Big\}. \label{eq:base-cond-2}
\end{align}
The following results \cite{sosmnp} relate a base vector $x$ of function $f$, and a set of base vectors $y_\cs$ of a $f_\cs$:
\begin{lemma}
\label{lem:sum_vec}
Let $x(S) = \sum_{\cs} y_\cs(\cs \cap S)$ where each $y_\cs$ belongs to base polyhedra $B(f_\cs)$. Then the vector $x$ belongs to base polyhedron $B(f)$.
\end{lemma}
\begin{lemma}
\label{lem:sum_vec_converse}
Let $x$ be a vector belonging to the base polyhedron $B(f)$. Then, $x$ can be expressed as the sum: $x(S)=\sum_{\cs} y_\cs(S \cap \cs)$, where each $y_\cs$ belongs to the submodular polyhedron $B(f_\cs)$ i.e., $y_\cs \in B(f_\cs) ~ \forall ~ \cs$.
\end{lemma}
The block coordinate descent approach based on the results requires each block to represent a base vector $y_\cs$ as defined above (c.f. \cite{sosmnp}). Note that a base vector $y_\cs$ is of dimension $|\cs|$, whereas a base $x$ is of dimension $|\VS|$.  Since $|\cs| \ll |\VS|$, minimizing the norm of $y_\cs$ over its submodular polyhedron $B(f_\cs)$ is much more efficient than minimizing the norm of $x$ by just applying the MNP algorithm. 
%
%
%
However, the algorithm based on the above fails to generalize on submodular functions defined on lattice.

%

	\subsection{Distributive lattices}
\label{sec:DL}
It is not feasible to list all the elements of distributive lattice $\DS$. Now, we shall see how to represent $\DS$ as a structured system as shown in \cite{fujishige2005submodular}. For each element $e \in \VS$ define:

\begin{align}
    \DS (e) = \bigcap \{ X | e \in X \in \DS\}.
    \label{eq:defineD}
\end{align}

$\DS (e)$ is the unique state with minimum number of elements which contains $e$. We call $\DS(e)$, the minimal covering state of an element $e \in \VS$. Note that for $e' \in \DS(e)$, $e'$ is contained in $\DS(e)$ (or $\DS(e)$ covers $e'$) but $\DS(e)$ may or may not be minimal cover of $e'$. Therefore $D(e')$ will be a subset of $\DS(e)$. We have, 
\begin{align}
    \DS(e') \subseteq \DS(e), \quad e'\in \DS(e).
    \label{eq:coveringsubset}
\end{align}
Also define the directed graph $\GS(\DS) = \{ \VS , \BS(\DS)\}$ with vertex set $\VS$ and arc set $\BS(\DS)$ given as

\begin{align}
   \BS(\DS) = \{ (e, e') | e \in \VS , e' \in \DS(e)\}.
   \label{eq:defineedges}
\end{align}

Note that if two elements have identical minimal covering state then such a pair has bidirectional edges. By using induction we can conclude that a set of elements which have identical minimal covering states, form a strongly connected component in $\GS(\DS)$. And such a set of elements is the subset of the minimal covering state. We show an example in figure \ref{fig:poset} to make the matter more clear. In the following lemma, we show that elements in strongly connected components co-occur together in a state.

\begin{figure}
  \centering
	 \includegraphics[width=\linewidth]{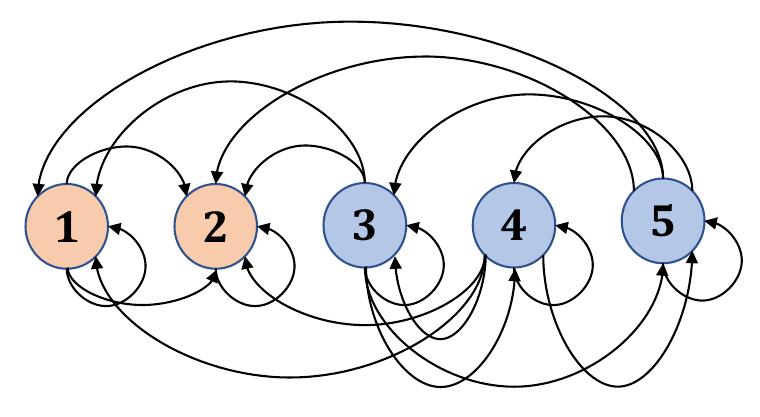}
	  \caption{An example of $\GS(\DS)$ corresponding to lattice $\BS(\DS) = \{ \phi ,  \{1,2\}, \{1,2,3,4,5\} \}$ for $\VS = \{1,2,3,4,5 \}$. Note that $\DS(1) = \DS(2) = \{1,2\}$ and $ \DS(3) = \DS(4) = \DS(5) = \{1,2,3,4,5\}$. We can see that elements with identical minimal covering state forms strongly connected components thus we get two strongly connected componentss $\{1,2\}$ and $\{3,4,5\}$. Also note that there is a natural ordering induced for vertex in one component to another component.}
	  \label{fig:poset}
\end{figure}

\begin{lemma}
Consider a set $\ES$ of elements forming a strongly connected components in $\GS$. And an arbitrary state $X \in \DS$. For an element $e \in X$,  if $e \in \ES$ then every element of set $\ES$ is in $X$ , i.e. $\exists$ $a\in X$,  $\forall a \in \ES$.
\label{lemma:posetelements}
\end{lemma}
\begin{proof}
By eqn \ref{eq:defineD}, for an element $e \in \VS$, we have, $\DS(e)$ is the intersection of all states in $\DS$ containing $e$. Therefore all the states which contain $e$ also contains every element in $\DS(e)$ and we know $\ES \subset \DS(e)$. Therefore, every element of $\ES$ will be in $X$.    
\end{proof}

Now we prove if there is an edge from a vertex $e$ in a strongly connected component to a vertex $e'$ in another strongly connected component then from each vertex in former strongly connected component there is an edge to $e'$.  

\begin{lemma}
For any two strongly connected components $S_1$ and $S_2$ in $\GS$. If a vertex $e \in S_1$ has an edge $(e,e')$ where $e' \in S_2$. Then there exist edge $(a,e')$, $\forall a \in S_1$.   
\label{lemma:ccedges}
\end{lemma}
\begin{proof}
  By eqn \ref{eq:defineD}, an edge $(e,e')$ implies $e' \in \DS(e)$. We know that all elements in a strongly connected component have identical minimal covering state. Which implies that $e' \in \DS(a)$, $\forall a \in S_1$ or there exist edge $(a, e')$ $\forall a \in S_1$. 
\end{proof}

Now we shall see that any two strongly connected components can not have bidirectional edges.

\begin{lemma}
For any two strongly connected components $S_1$ and $S_2$ in $\GS$. if there exist an edge from a vertex in $S_1$ to a vertex in $S_2$ then there is no edge from $S_2$ to $S_1$.
\label{lemma:unidirectionscc}
\end{lemma}
\begin{proof}
We prove this by contradiction. Let us say there exist an edge form a vertex in $S_1$ to a vertex $e'$ in $S_2$ and from a vertex in $S_2$ to a vertex $e$ in $S_1$. By lemma \ref{lemma:ccedges}, we have that if a vertex $e'$ in $S_2$ has an edge incident on it from a vertex in $S_1$ then each vertex in $S_1$ must have an edge incident on $e'$. Similarly, by the same argument, there would be edges from each element in $S_2$ incident on vertex $e$ in $S_1$. Now we have that elements $e,e'$ have bidirectional edges between them. Which implies that element in different strongly connected components forms a strongly connected component which is a contradiction. 
\end{proof}

Decompose the graph $\GS(\DS)$ into strongly connected components $\HS_i = (\ES_i , \BS_i)$, $(i \in \IS)$. Here $\ES_i$ denotes the vertex set and $\BS_i$ denotes the set of edges. By lemma \ref{lemma:unidirectionscc}, we have that strongly connected components have only uni-directional or no edges between them. It induces a partial order $\oe_\DS$ between strongly connected components. For example, $\HS_i \oe_\DS \HS_j$ for $i,j \in \IS$ if and only if there exists a directed path from a vertex of $\HS_j$ to a vertex $\HS_i$. We call an ordering $\o$ obeys partial order $\oe_\DS$ if for every pair $\HS_i, \HS_j$,  $\HS_i \oe_\DS \HS_j$ we have $\ES_i \o \ES_j$. Define the set of the vertex sets $\ES_i (i\in I)$ of the strongly connected components by, 

\begin{align}
    \poset = \{ \ES_i | i \in \IS \},
\end{align}

$\poset$ is a partition of $\VS$. Now, we obtain a partially ordered set also known as poset $ \PS(\DS) = ( \poset , \oe_\DS)$.

\begin{definition}[Simple Distributive Lattice]
\label{def:simplelattice}
Distributive lattice $\DS$ is called simple if the partition $\DS$ is composed of singletons of $\VS$ alone, i,e., $ \poset = \{ \{e\} | e \in \VS\}$.
\end{definition}

A submodular system, denoted as tuple $(\DS, f)$, defines submodular function $f$ over distributive lattice $\DS$. For any two sets $A,B \in \DS$, $f$ satisfies the following,
\begin{align}
    f(A) + f(B) \geq f(A \cap B) + f(A \cup B).
\end{align}

Submodular system $(\DS, f)$ is called simple if $\DS$ is simple. For a non-simple submodular system there is at least one element $\ES \in \poset$ that would have at least two elements in it, i.e., $|\ES| \geq 2$. By lemma \ref{lemma:posetelements}, for each state $X \in \DS$, either $X$ would contain all elements in $\ES$ or no element of $\ES$ would be in $X$. Which implies that we can tie all the elements of $\ES$ together and define a new submodular system which is defined on set $\poset$ instead of $\VS$. More formally for non-simple submodular system $(\DS, f)$ on $\VS$, define, 

\begin{align}
    \Xhat  = \{Z | Z \in \poset , Z \subseteq X \} (X \in  \DS), \\
    \hat{\DS} = \{ \Xhat | X \in \DS\}, \\
    \hat{f}(\Xhat) = f(X)  ( X \in \DS).
\end{align}

Then we have a simple submodular system $(\hat{\DS}, \hat{f})$ on $\poset$. Therefore without loss of generality we consider only the simple submodular system in the following sections. We would like to point out that base polyhedron is unbounded when $\DS$ is not complimented even for a simple submodular system \cite{fujishige2005submodular}. Now, we shall see how we can extend the submodular function defined over a distributive lattice to boolean lattice. Later show how the unbounded extended space can be handled very efficiently.  


	\section{Distributive Lattice to Boolean Lattice transformation}
\label{sec:tranformation}
Let $F: \DS \rightarrow \real$ be a submodular function  defined over the distributive lattice $\DS \subseteq 2^{\VS}$ then,
\begin{align}
F(X) + F(Y ) \geq F(X \cap Y ) + F(X \cup Y ) \quad \forall (X,Y \in \DS).
\end{align}

\noindent We denote $S \subseteq \VS$, as a \underline{\emph{state}}. We call states in distributive lattice $S \in \DS$ as \underline{\emph{valid states}} and states, $S \in 2^\VS - \DS$, not in $\DS$ as \underline{\emph{invalid states}}. Now we extend the notion of validity for ordering/extreme base. Intuitively, if an extreme base (using edmond's greedy algorithm) is computed over only valid states then we call it valid extreme base and its corresponding ordering is called valid ordering. More formally we state it as,   
\begin{definition}[Valid Ordering/Extreme Base]
\label{def:validinvalid}
An ordering $\o$ is called a \underline{valid ordering}, if it obeys partial order $\oe_\DS$.  An extreme base $\ebo$ is called a \underline{valid extreme base}, if it corresponds to a valid ordering.
\end{definition}
\noindent The orderings or extreme-bases which are not valid are called \emph{invalid}.


\begin{definition}[Covering State, Minimal Covering State]
\label{extend}
For an arbitrary state, $S \in \VS$, a valid state, $\hat{S} \in \DS$, is called covering if $S \subseteq \hat{S}$. There may be multiple covering states corresponding to a $S$. The one with the smallest cardinality among them is referred to as the minimal covering state, and is denoted by $\Sbar$. There is a unique minimal covering state corresponding to any $S$. For a valid state $S=\Sbar$.
\end{definition}

We are now ready to define a set function $f$ over boolean lattice $2^{\VS}$ which is submodular and identical to the submodular function $F$ on valid states.
\begin{definition}[The Extended Submodular Function] \label{definef} 
\begin{align*}
f(S) = \begin{cases}
F(S), & \text{if } S \in \DS \\
f(\Sbar) + (|\Sbar| - |S|)L & \text{if } S \in 2^\VS \setminus \DS \\
\end{cases}
\end{align*}
Here $L \gg  M = \left[\max_{S \in \DS} f(S) - \min_{S \in \DS} f(S)\right]$.
\end{definition}

We have already established that we can work with the simple submodular system without loss of generality. We know (by definition \ref{def:simplelattice}) that there are no strongly connected components in graph $\GS(\DS)$ if $\DS$ is simple. 

\begin{definition}[Segments of  $\GS(\DS)$]

Decompose the graph $\GS(\DS)$ into connected components $\GS_i = (\VS_i , \AS_i)$. Where $\VS_i$'s are vertex sets and $\AS_i$'s are edge sets. Note that (by lemma \ref{lemma:unidirectionscc}) no two connected components will have an edge between their vertex. We define the set of vertex sets $\VS_i$'s of the connected component by,

\begin{align}
    \partition = \{\VS_i | i \in \IS\}
\end{align}

Note that $\partition$ (ISHANT's COMMENT: $\partition$ is replacement of set of pixels) is a partition of $\VS$. We denote an element of $\partition$ as \underline{\emph{segment}}. 

\end{definition}

Prior to our work the notion of partial order and decomposition of graph $\GS(\DS)$ has been discovered by Birkhoff in his seminal paper \cite{birkhoff1937rings}. In our work, we shall see how one can leverage such decomposition in the corresponding base polyhedron of submodular function defined on the lattice. 

Note that for a valid ordering, elements corresponding to each segment $\VS_i \in \partition$ has an order defined by partial order $\oe_\DS$. Let $d_i = |\VS_i|$ be the cardinality of segment $\VS_i$ and $n$ (ISHANT's COMMENT: $n$ is replacement of number of pixels) be the number of connected components in $\GS(\DS)$, i.e., $n = |\partition|$. Let $\VS_i$ be represented by $\VS_i = \{p_i^1,p_i^2  ~ \cdots ~ p_i^{d_i}\}$. Using an arbitrary ordering of segments $\VS_i$'s, we define a total order over all vertex set $\VS$, called \underline{\emph{universal ordering}}. 

\begin{definition}[Universal Ordering]
	\label{def:univ-ordering}
	Assuming an arbitrary ordering among the sets $\{ \VS_i | i \in \IS\}$, the \underline{universal ordering}, defines a total ordering of the elements $\{p |  p \in \VS_i\} (i \in \IS)$, 
	\[
		\o_0: p_1^1 ~ \o \cdots \o ~ p_1^{d_1} \o \cdots \o ~ p_n^1 \cdots ~ \o ~ p_n^{d_n}.
	\]
\end{definition}

It is easy to see that $f(S)$ can also be defined as follows:
\begin{definition}[The Extended Binary Set Function: Alternate Definition] \label{definef_alt}
\begin{align} 
	f(S) =  f(\Sbar) + \sum_{p \in \partition}(|\Sbar_p| - |S_p|)L,
\end{align}
where $\Sbar_p = p \cap \Sbar$, and $S_p = p \cap S$.
\end{definition}
\begin{theorem} \label{thm:f-submodularity}
	The extended binary set function $f$, as given by Definition \ref{definef}, is submodular, and $\min f(\cdot) = \min F(\cdot)$. 
\end{theorem}

Recall that we define the extended binary submodular function for the valid states as equal to the original  function $F$ and for the invalid states as the following:
\[ f(S) =  f(\Sbar) + (|\Sbar| - |S|)L. \] 
Here $\Sbar$ is the minimum covering state for an invalid state, $S$, which is defined as the smallest cardinality valid state, $\Sbar \in \DS$, such that $S \subset \Sbar$. For a valid state $S=\Sbar$.

Let us factorize $f(S)=g(S)+h(S)$, where $h(S) = f(\Sbar) +  |\Sbar| L$, and $g(S) = -|S| L$. Since, $g$ is modular, it is sufficient to show that $h$ is submodular. We will need the following result to prove the Theorem.

\begin{lemma}
	\label{lemma:invalid_state_ineq}
	For sets $X,Y$, and $(X \cap Y) \subseteq \VS$ and their minimum covering states $\Xbar, \Ybar$, and $\overline{X \cap Y}$ respectively: 
	\[ f(\overline{X \cap Y}) \leq f(\Xbar \cap \Ybar)) \]
\end{lemma}

\begin{proof}
	Recall that for any valid state $S$, $\Sbar=S$. Consider, two valid states $A,B \subseteq V$ with $A \subseteq B $. It is easy to see that:
	\be h(B) - h(A) = L(\card{B}-\card{A})+ f(B)-f(A) \geq 0. \label{eq:app_ext_func_1} \ee
	Since, $A \subseteq B $, therefore, $\card{B}-\card{A} \geq 0$. Also $L \gg f(B)-f(A)$ by definition. Therefore, $h(A) \leq h(B)$. Further, it has been shown in section $6$ of \cite{schrijver00} that for two $X,Y \in \VS $, $\overline{X \cup Y} = (\Xbar \cup \Ybar)$ and $\overline{X \cap Y} \subseteq (\Xbar \cap \Ybar)$ holds. Therefore, using Eq. (\ref{eq:app_ext_func_1}), $f(\overline{X \cap Y}) \leq f(\Xbar \cap \Ybar))$.
\end{proof}

We can now give the proof of the theorem as follows. For the valid states, the extended function $f$, has been shown to be submodular in \cite{mlgc}. Therefore, here, we show only for the cases when $S$ is an invalid state. Now, for two arbitrary (valid or invalid) sets, $X,Y \subseteq \VS$ 

\begin{align*}
	h(X) + h(Y) & = f(\Xbar) + f(\Ybar) + |\Xbar| L + |\Ybar|L 
	\\
	& \geq f(\Xbar \cup \Ybar) + f(\Xbar \cap \Ybar) + |\Xbar| L + |\Ybar| L 
	\tag{Using submodularity over $\Xbar$, and $\Ybar$} 
	\\
	& = f(\Xbar \cup \Ybar) + f(\Xbar \cap \Ybar) + |\Xbar \cup \Ybar| L  
	+ |\Xbar \cap \Ybar|L 
	\tag{Since $|\Xbar| + |\Ybar| = |\Xbar \cup \Ybar| + |\Xbar \cap \Ybar|$}
	\\
	& = f(\overline{X \cup Y}) + |\overline{X\cup Y}| L + f(\Xbar \cap \Ybar) + |\Xbar\cap \Ybar| L 
	\tag{Since $\overline{X \cup Y} = (\Xbar \cup \Ybar)$}
	\\
	& \geq f(\overline{X \cup Y}) + |\overline{X\cup Y}| L + f(\overline{X \cap Y}) +  
	|\overline{X\cap Y}|L 
	\tag{Using Lemma \ref{lemma:invalid_state_ineq}}
	\\
	& = h(X \cup Y) + h(X \cap Y).
\end{align*}

The above shows that $h$ is submodular. It is easy to see that, $g$, as defined above is modular. Since addition of a modular function and a submodular function is submodular, therefore, $f = g + h$ is submodular.

Note that in the proposed extension, any value of $L \gg  M$, keeps the function, $f$, submodular. However, as we show later, choosing such a large value of $L$, makes the contribution of some extreme bases very small causing precision issues in the computation. We also show that including those extreme bases with very small contribution is extremely important for achieving the optimal inference. The major contribution of this paper is in showing that one can perform an efficient inference bypassing $L$ altogether. Therefore, the use of $L$ is merely conceptual in our framework. There is no impact of actual value of $L$ on the algorithm's performance.


	\section{Representing Invalid Extreme Bases}
\label{sec:invalidbase}
In the discussion that follows, we refer to any scalar as \ue{small} or \ue{finite} if its absolute value is $\ll L$, and \ue{large} or \ue{infinite} if the absolute value is $\propto L$. We write Eq. (\ref{eqn:convexcombination}) as:
\be 
x = x_v + x_i = \sum_{b^{\o_j} \in R} \lambda_j b^{\o_j} + \sum_{b^{\o_i} \in Q} \lambda_i b^{\o_i}.
\label{eq:base_split_valid_invalid} 
\ee
Here, $R$ and $Q$ are the sets of valid and invalid extreme bases, and $x_v$, and $x_i$, their contribution in $x$ respectively.  It is easy to see that, all the elements of $x_i$ must be much smaller than $L$ \footnote{We start the algorithm with a valid extreme base, where the condition is satisfied. In all further iterations the norm of $x$ decreases monotonically, and the condition continues to remain satisfied.}.  We first focus on the relationship between $\lambda$ and $L$ in the block of invalid extreme bases.
\begin{lemma}
	\label{lemma:elementbound}
	For any element, $e$, of an invalid extreme base, $\ebo: \ebo(e) = a_e L + b_e$, where $|a_e|, |b_e| \ll L$ and $a_e \in I$.
\end{lemma}

\begin{proof}
Let $S_2$ be the set of all elements smaller than $e$ as per $\o$. Let $S_1=S_2 \cup \{e\}$.
\begin{align*}
	\ebo(e) &= f(S_1)-f(S_2)   \tag{Definition of extreme base}  \\
	& = \left( f(\Sbar_1)+ (|\Sbar_1| - |S_1|) L \right) 
	  - \left( f(\Sbar_2)+ (|\Sbar_2| - |S_2|) L \right)   \tag{Definition \ref{definef}}  \\ 
	& = \left( f(\Sbar_1) - f(\Sbar_2) \right) 
	+ \left( |\Sbar_1| - |S_1| - |\Sbar_2| + |S_2| \right) L \\
	& = a_e L + b \tag{where $|a_e|, |b_e| \ll L$}
\end{align*}
\end{proof}

\begin{lemma}
	\label{lamba-for-invalid-eb}
	Consider two base vectors $x_1$ and $x_2$ such that $\norm{x_1}^2, \norm{x_2}^2 < \card{\VS}M^2$. If $x_2 = (1-\lambda) x_1 + \lambda \ebo$ and $\ebo$ is an invalid extreme base, then $\lambda \leq \card{\VS}\frac{M}{L}$.
\end{lemma}

\begin{proof}
Recall that in our algorithm, base vector is represented as the sum of contributions from valid and invalid extreme bases separately: $x=x_v+x_i$, where $x_v$ and $x_i$ are the base vectors collecting contributions of valid and invalid extreme bases respectively. Further, we start from a valid extreme base and in each iteration of the algorithm, keep on decreasing the norm of the overall base vector. Note that, all the elements of a valid extreme base are smaller than $M$. Therefore the squared $\ell_2$ norm of the overall base vector is less than $|\VS|M^2$ at any point in the algorithm. 

We will prove the lemma by contradiction, and show that unless the $\lambda$ for the invalid extreme base is less than $\card{\VS}M/L$, the squared norm of the overall base vector is more than $|\VS|M^2$, which is a contradiction. 

We will first need to prove the following result:

\begin{lemma} \label{lemma:elementbound_inner}
	Consider an invalid ordering $\o$,  and its corresponding invalid extreme base $\ebo$. Let $e$ be the smallest element (as per $\o$), for which validity condition is violated. Then, $\exists ~ a_e \in \real$, and $a_e \geq (1-M/L)$, s.t. $\ebo(e) = a_e L$. 
\end{lemma}

\begin{proof}
	Let $S_2$ be the set of all elements smaller than $e$ as per $\o$. Let $S_1=S_2 \cup \{e\}$. Notice that $S_2$ is a valid and $S_1$ is an invalid state. 

	\begin{align*}
	\ebo(e) &= f(S_1)-f(S_2)   \tag{Definition of extreme base}  \\
	&= \left( f(\Sbar_1)+ (|\Sbar_1| - |S_1|) L \right) - f(S_2)   \tag{Definition \ref{definef}}  \\ 
	& \geq \min_{S \in Z}f(S)-f(S_2) + (|\Sbar_1| - |S_1|)L  \tag{$\Sbar_1$ is a valid state, therefore $f(\Sbar_1)\geq \min_{S \in Z} f(S)$} \\
	&\geq \min_{S \in Z}f(S)-\max_{S \in Z}f(S) + (|\Sbar_1| - |S_1|)L   \tag{$S_2$ is a valid state, therefore $f(S_2)\leq \max_{S \in Z} f(S)$} \\
	&= (|\Sbar_1| - |S_1| - M/L)L  \tag{Defintion of $M$} \\
	\end{align*}
	
	Note that for any invalid state $S_1$, $(|\Sbar_1| - |S_1|) \geq 1$. Therefore there exists $a_e \geq (1-M/L)$ such that $\ebo(e) = a_e L$.  
\end{proof}

To prove our main result by contradiction, assume $\lambda > \card{\VS}M/L$. Let $e$ be the smallest element (as per $\o$ of invalid extreme base $\ebo$), for which validity condition is violated. Consider:

\begin{align*}
	(x_{2}(e))^2 	&=((1-\lambda)x_1(e)+\lambda \ebo(e))^2 
	\\
	& > ((1-\lambda) x_1(e)+ \ebo(e)  |\VS| M/L)^2 \tag{$\lambda > \card{\VS}M/L$}
	\\
	& = ((1-\lambda) x_1(e)+ a_e  |\VS| M))^2. \tag{Using lemma \ref{lemma:elementbound_inner}} 
    \intertext{
    	Two cases are possible: 
    	\begin{enumerate} \item $x_1(e) \geq 0$: \end{enumerate}
    } 
	(x_{2}(e))^2  &\geq ((1-\lambda) x_1(e)+ a_e |\VS| M))^2 
	\\
	& \geq (a_e |\VS| M)^2   \tag*{(Since $(1-\lambda) \geq 0$)}
	\\ 
	& = (a_e|\VS|) (|\VS|M^2) 	 
	\\ 
	\intertext{
		Since $M \ll L$, and $a_e \geq (1-M/L)$, therefore $a_e \approx 1$. The smallest problem size that we consider is $|\VS|=6$. Hence, for our case, $a_e|\VS| > a_e \sqrt{|\VS|} > 2$. This implies: 
	}
	(x_{2}(e))^2 & >  |\VS| M^2.  
	\\ 
	\intertext{
		\begin{enumerate} \setcounter{enumi}{1} \item $x_1(e) < 0$: \end{enumerate}
		Note that for $x_1$ and any element $e \in \VS$ we have $x_1(e)^2 \leq \norm{x_1}^2 \leq  |\VS|M^2$. This implies that $x_1(e) \geq -\sqrt{|\VS|} M$.
	} 
	(x_{2}(e))^2  &\geq ((1-\lambda) x_1(e)+ a_e  |\VS| M)^2 
	\\
		& \geq (-(1-\lambda)\sqrt{|\VS|} M+ a_e  |\VS| M)^2 
		\\
		& \geq (-\sqrt{|\VS|} M+ a_e |\VS| M)^2  \tag{Since $1 \geq (1-\lambda) \geq 0$}
		\\
		& = M^2|\VS|(a_e\sqrt{|\VS|}-1) \tag{As described in the first case $a_e\sqrt{|\VS|}> 2$}	 
		\\ 
		& > |\VS|M^2. 		  
	\end{align*}
Both the cases imply that if $\lambda > |\VS|M/L$ then norm $||x_2||^2 > |\VS|M^2$ which is a contradiction. Hence for any invalid extreme base its contribution $\lambda$ in the overall base vector must be less than $|\VS|M/L$.
\end{proof}

\noindent Conceptually, Lemma \ref{lemma:elementbound} shows that all elements of an invalid extreme base are either small or are proportional to $L$ (and not proportional to, say $L^2$, or other higher powers of $L$). Whereas, Lemma \ref{lamba-for-invalid-eb} shows that since $\VS$ and $M$ are effectively constants, $\lambda$ the multiplicative factor associated with in the contribution of invalid extreme bases, $\lambda$ is proportional to $1/L$. Therefore, for $L \approx \infty$, the value of $\lambda \approx 0$. However, it is important to note that the value of $\lambda \ebo(e)$, is always finite. It is easy to see that, whenever $a_e=0$, $\lambda \ebo(e) \approx 0$, and when $a_e \neq 0$, the $L$ present in the $\ebo(e)$ and $1/L$ present in $\lambda$ cancel each other, leading to a finite contribution. The argument as given above motivates our overall approach in this paper that, for a numerically stable norm minimization algorithm, focus should be on manipulating the finite valued product $\lambda \ebo$, and not the individual  $\lambda$ and $\ebo(e)$. We show in the following sections that this is indeed possible. 

We start by showing that it is possible to find a small set of what we call \emph{elementary} invalid extreme bases whose linear combination contains as a subset the space of vectors $x_i$ as given in Eq. (\ref{eq:base_split_valid_invalid}). Crucial to doing this is the notion of \emph{canonical orderings}. 

\subsection{Canonical Ordering and Its Properties}

\setlength{\columnsep}{10pt}%
\setlength{\intextsep}{1pt}%
\begin{wrapfigure}{r}{0.6\linewidth}
	\begin{center}
		\includegraphics[width=\linewidth]{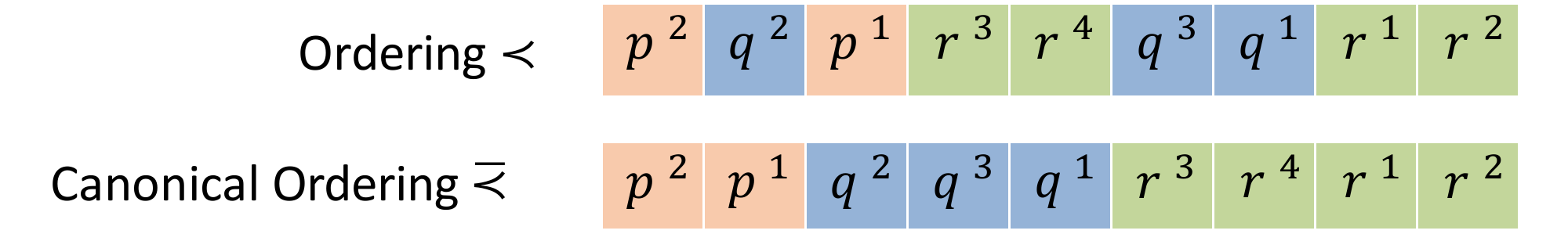}
	\end{center}
	\caption{Top: An ordering defined over segments $\partition=\{p,q,r\}$, where $p = \{p^1, p^2\}$, $q = \{q^1, q^2, q^3\}$, $r = \{r^1, r^2, r^3,r^4\}$. Bottom: Corresponding canonical ordering.}
	\label{Fig:canonical}
\end{wrapfigure}
In an arbitrary, valid or invalid, ordering $\o$ consider two adjacent elements $u$ and $v$ such that $u ~ \o ~ v$. We term swapping of order locally between $u$ and $v$ in $\o$ as an \emph{exchange operation}. The operation will result in a new ordering $\o_\text{new}$ such that $u$ and $v$ are still adjacent but $v ~ \o_\text{new} ~ u$.

Consider a strategy in which starting with $\o$ we carry out exchange operations till all the elements corresponding to each segment $\VS_i,  (i \in \IS)$ come together. Note that we do not change the relative ordering between elements in the same segment $\VS_i$,  $(i \in \IS)$. We call the resultant ordering the \emph{canonical} form of the original ordering $\o$ and denote it by $\ob$. The corresponding extreme base is called \emph{canonical extreme base}. Note that there can be multiple canonical forms of an ordering. Figure  \ref{Fig:canonical} contains an example of an arbitrary ordering and one of its canonical orderings. We emphasize here that there may be more than one canonical orderings corresponding to $\o$. 

Note that a valid / invalid ordering leads to a valid / invalid canonical ordering. Further, a canonical ordering is agnostic to any relative order among segments. For example, for segments $\VS_i \in \partition, (i \in \IS)$, a canonical ordering only requires that all elements of each segment $\VS_i, (i \in \IS)$ are contiguous. In general a canonical ordering $\ob$ corresponding to a $\o$ can be any one of the possible canonical orderings.

\begin{lemma}  \label{lemma:ebequal}
	Let $\o$ be an invalid ordering and $\ob$ be its canonical ordering. Then, $\ebo(e) - \eb^{\ob}(e) \ll L, \forall e \in \VS$.
\end{lemma}

\begin{proof}
	Let $S$ and $S'$ be the set of elements preceding $p^i, p\in \partition,$ in ordering $\o$ and $\ob$ respectively. Consider the term $\ebo(p^i)$:
	\begin{align*}
	\ebo(p^i) & = f(S \cup \{p^i\}) - f(S)
	\\
	& = L \sum_{q \in \partition} \Big( |(\overline{S \cup \{p^i\}})_q| - |(S \cup \{p^i\})_q| \Big) 
	  + f\Big( \overline{S \cup \{p^i\}} \Big)  \\
	& \quad  - L \sum_{q\in \partition} \Big( |\Sbar_q| - |S_q| \Big) - f\Big( \Sbar \Big) 
	  \tag{Def. \ref{definef_alt}} 
	\\ 
	& = L \Big( |\overline{S_p \cup \{p^i\}}| - |S_p \cup \{p^i\}| \Big) 
	  - L \Big( |\Sbar_p| - |S_{p}| \Big)	 \\
	 & \quad + f(\overline{S \cup \{p^i\}}) - f(\Sbar)
	\intertext{
		Similarly we obtain:
	}
	b^{\ob}(p^i) & = L \Big( |\overline{S'_p \cup \{p^i\}}|- |S'_p \cup \{p^i\}| \Big) 
	- L \Big( |\overline{S'}_p| - |S'_{p}| \Big)  \\
	& \quad + f \Big( \overline{S' \cup \{p^i\}} \Big) - f \Big( \overline{S'} \Big)
	\intertext{
		Note that a canonical ordering does not change intersay ordering between elements corresponding to a particular segment. Therefore, $S_p=S'_p$, and:
	}
	\ebo(p^i) - b^{\ob}(p^i) & = (f(\overline{S \cup \{p^i\}}) - f(\Sbar)  
	- f(\overline{S' \cup \{p^i\}}) +  f(\overline{S'}) 
	\end{align*}
	Since, all terms in the r.h.s. of the equation above, correspond to valid sets, therefore $\ebo(p^i) - b^{\ob}(p^i)  \ll L$.
\end{proof}

The above result serves to indicate that by changing an invalid extreme base to canonical one, the change in value of any element of the extreme base is much less than $L$. Therefore, due to Lemma \ref{lamba-for-invalid-eb}, one can conclude that the contribution of an invalid extreme base or its canonical extreme base in a base vector is going to be the same. 
\begin{lemma}
\label{lemma:exchange}
For a canonical invalid ordering $\ob$, let $p^i$ and $p^j$ be two adjacent elements corresponding to a segment $p \in \partition$, s.t. $p^i ~ \ob ~ p^j$. Let $\ob_p^{i,j}$ be the ordering obtained by swapping $p^i$ and $p^j$. Then: $\eb^{\ob_p^{i,j}} - \eb^\ob = (\chi_p^j-\chi_p^i) (a L + b)$, where $\chi_p^i$ is an indicator vector for the element $p^i$, and $a,b \ll L$. 
\end{lemma}

\begin{proof}
	Recall that for an extreme base $b^\o: b^\o(k) = f(k_\o) - f((k-1)_\o)$, where $k_\o$ is the first $k$ elements in the ordered set $\{v_1, \ldots,v_k, \ldots, v_n\}$. Since the swap between $p^i$, and $p^j$ leaves the set of preceding elements unchanged for all other elements, therefore, $\eb^{\ob_p^{i,j}} - \eb^\ob$ is non-zero corresponding to only $p^i$, and $p^j$. 
	 	
	Let $S$ be the set of elements preceding $p^i$ in $\ob$. Now:
\begin{subequations}
	\label{eq:lemma_exchange}
	\begin{align}
	\eb^\ob(p^j) &=  f(S \cup\{p^j\} \cup \{p^i\}) - f(S \cup \{p^i\}) 
	\nonumber 
	\\
	& = L ( |\overline{S \cup \{p^j\} \cup \{p^i\}}|- |S \cup \{p^j\} \cup \{p^i\}| ) + f(\overline{S \cup \{p^j\} \cup \{p^i\}}) 
	\nonumber 
	\\
	& - L ( |\overline{S \cup \{p^i\}}| - |S \cup \{p^i\}|) - f(\overline{S \cup \{p^i\}}) \label{lemma-exchange-eq1m} 
	\intertext
	{
		Similarly we have,
	} 
	\eb^{\ob_p^{i,j}}(p^j) &= L ( |\overline{S \cup \{p^j\}}|- |S \cup \{p^j\}| ) - L ( |\Sbar| - |S|)  + f(\overline{S \cup \{p^j\}}) - f(\Sbar), 
	\label{lemma-exchange-eq2m} 
	\intertext
	{ 
		Subtracting Eq. (\ref{lemma-exchange-eq1m}) from Eq. (\ref{lemma-exchange-eq2m}) and using $( |S \cup{\{p^i\}}| + |S \cup{\{p^j\}}| - |S \cup{\{p^i\}} \cup \{p^j\}|-|S|) = 0$ we have
	}
	b^{\ob_p^{i,j}}(p^j) - b^{\ob}(p^j)
	& = L ( |\overline{S \cup \{p^i\}}| + |\overline{S \cup \{p^j\}}|  
	- |\overline{S \cup \{p^i\} \cup \{p^j\}}| - |\overline{S}|) 
	\nonumber 
	\\
	& + ( f(\overline{S \cup \{p^j\}}) - f(\Sbar) - f(\overline{S \cup \{p^j\} \cup \{p^i\}}) + f(\overline{S \cup \{p^i\}}), 
	\nonumber 
	\\
	& = aL+b, \nonumber
	\end{align}
\end{subequations}
where:
\begin{align*} 
	a & = |\overline{S \cup \{p^i\}}| + |\overline{S \cup \{p^j\}}|  
	- |\overline{S \cup \{p^i\} \cup \{p^j\}}| - |\overline{S}|, \text{ and} \\ 
	b & = ( f(\overline{S \cup \{p^j\}}) - f(\Sbar) - f(\overline{S \cup \{p^j\} \cup \{p^i\}}) + f(\overline{S \cup \{p^i\}}).
\end{align*} 
Note that $b$ is sum of function values at valid states and is $\ll L$. Two cases arise for the value of $a$:
\begin{enumerate}
\item $i<j$: \\ 
In this case $|\overline{S_p\cup{\{p^i\}}\cup{\{p^j\}}}| = |\overline{S_p\cup{\{p^j\}}}|$, and $a=|\overline{S \cup \{p^i\}}| - |\overline{S}|$. Therefore, $a \ll L$.
\item $j<i$: \\
In this case $|\overline{S_p\cup{\{p^i\}}\cup{\{p^j\}}}| = |\overline{S_p\cup{\{p^i\}}}|$, and $a=|\overline{S \cup \{p^j\}}| - |\overline{S}|$. Therefore, $a \ll L$
\end{enumerate}
Hence $b^{\ob_p^{i,j}}(p^j) - b^\ob(p^j) = a L + b$, such that $a, b \ll L$. Further, since $b^{\ob_p^{i,j}}$ and $b^\ob$ are extreme bases, and the sum of all the elements in them is constant, therefore, the reverse must hold for $b^{\ob_p^{i,j}}(p^i) - b^\ob(p^i)$. Hence $\eb^{\ob_p^{i,j}} - \eb^\ob = (\chi_p^j-\chi_p^i) (a L + b)$
\end{proof}

Lemma \ref{lemma:exchange} relates the two extreme bases when one pair of their elements is swapped. It is useful to note that in a valid extreme base all elements have small values. With each swap in an invalid canonical ordering we either move the canonical ordering towards validity or away from it. In each swap  the change in the value of an element is proportional to $L$ (positive or negative). Since conversion of an invalid canonical ordering to a valid one may involve swaps between a number of elements, the extreme base corresponding to the invalid ordering will contain multiple elements with values proportional to $L$. The special cases are the ones in which only one swap has been done.  In these cases there will be only two elements with values proportional to $L$ (positive and negative). We show that using such extreme bases as the basis to represent canonical invalid extreme bases. In the next section we show that it is indeed possible. 

\subsection{Elementary Invalid Extreme Base}

\begin{definition}[Elementary Invalid Extreme Base]
The ordering obtained by swapping two elements $p^j$ and $p^{j+1}$, corresponding to a segment $p \in \partition$, in a canonical valid ordering, is called an \emph{elementary invalid ordering}. Its corresponding extreme base is called \emph{elementary invalid extreme base}, and is denoted as $b^{\ot_p^j}$.
\end{definition}

\begin{lemma} 
	\label{lemma:elem-eb-chi-rep}
 	Consider an elementary invalid extreme base $b^{\ot_p^i}$, obtained by swapping two adjacent elements $(p^{i},p^{i+1})$ in the universal ordering, $\o_0$ (Def. \ref{def:univ-ordering}). Then:
 	$ \eb^{\ot_p^i} - \eb^{\o_0} = (\chi_p^{i+1}-\chi_p^{i}) (L + b)$, 
 	where $\eb^{\o_0}$ is the valid extreme base corresponding to $\o_0$. 
\end{lemma}

\begin{proof}
	Recall:
	\begin{itemize}
		\item The universal ordered sequence, $\o_0$, which is a valid ordering, and also defines an arbitrary order among the segments. 
		\item The elementary invalid ordering, $\ot$, which is defined as the ordering obtained by making one swap between adjacent elements of a valid ordering (universal ordering). The corresponding extreme base is denoted as $b^\ot$.
	\end{itemize}
	Further, recall from the proof of Lemma \ref{lemma:exchange}, we showed that: $\eb^{\ob_p^{i,j}} - \eb^\ob = (\chi_p^j-\chi_p^i) (a L + b)$, such that $a=|\overline{S \cup \{p^i\}}| - |\overline{S}|$ (if $i<j$), or $a=|\overline{S \cup \{p^j\}}| - |\overline{S}|$ (if $j<i$). Now consider an elementary invalid extreme base $b^{\ot_p^i}$, obtained by swapping two adjacent elements $(p^{i},p^{i+1})$ in the universal ordering. The term $(\chi_{p}^{i+1} - \chi_{p}^{i})$ may be looked upon as corresponding to the creation of the elementary extreme base $b^{\ot_p^i}$ from $b^{\o_0}$. It is easy to see that for such special elementary invalid extreme bases created from universal ordering, $a=1$, and we have:
	\begin{align}
	\eb^{\ot_p^i} - \eb^{\o_0} = (\chi_p^{i+1}-\chi_p^{i}) (L + b) 
	\end{align}
	Hence, proved.
\end{proof}

\begin{lemma} 
\label{lemma:elem-eb-lin-comb}
An invalid canonical extreme base, $b^{\ob}$, can be represented as a linear combination of elementary invalid extreme base vectors such that:
$ b^{\ob} = \sum_{ p \in \partition} \sum_{i=1}^{m-1} \alpha^i_{p} b^{\ot_p^i} + \Lambda $,
where $0 < \alpha^i_p \ll L$, and $\Lambda$ is a vector with all its elements much smaller than $L$.
\end{lemma}

\begin{proof}
	Consider the canonical invalid ordering $\ob$ and let $\o_s$ be the starting canonical valid ordering from which it can be obtained by a series of swaps between adjacent elements. Note that since in the canonical ordering all the elements of each set $\VS_i (i \in \IS)$ are already together, therefore all the swaps required are between elements corresponding to same segments. Let us assume that total number of such swaps required are $T$. Starting from $\o_s$, let $\o_j$ represents the ordering obtained after $j$ such swaps. Hence, $\o_T=\ob$ by definition.  Let $j^\text{th}$ swap happens between elements $p^{k_j}$ and $p^{l_j}$, where $p \in \partition$.
	\begin{align}
	b^{\ob} - b^{\o_s} & = b^{\o_T} - b^{\o_s} 
	\nonumber \\
	& = \sum_{j=1}^T \Big( b^{\o_j} - b^{\o_{j-1}} \Big)
	\nonumber \\
	& = \sum_{j=1}^T \Big( \chi_p^{l_j} - \chi_p^{k_j} \Big)  (a_j L + b_j)  
	\tag{Using Lemma \ref{lemma:exchange}} \\
	& = \sum_{j=1}^T (\chi_{p}^{l_j} - \chi_{p}^{k_j}) a_j L 
	+ \sum_{j=1}^T (\chi_{p}^{l_j} - \chi_{p}^{k_j}) b_j.
	\nonumber 
	\intertext{
		Since $(\chi_{p}^{l_j} - \chi_{p}^{k_j}) = \sum_{i=l_j}^{k_j} (\chi_{p}^{i+1} - \chi_{p}^{i})$ we can write:
	}
	b^{\ob} - b^{\o_s} &= \sum_{j=1}^T a_j \sum_{i=l_j}^{k_j} (\chi_{p}^{i+1} - \chi_{p}^{i}) L  
	+ \sum_{j=1}^T b_j \sum_{i=l_j}^{k_j} (\chi_{p}^{i+1} - \chi_{p}^{i})
	\label{exchange:eq3} 
	\end{align}
	\noindent Recall from Lemma \ref{lemma:elem-eb-chi-rep}: 
	\begin{align*}
	& \eb^{\ot_p^i} - \eb^{\o_0} = (\chi_p^{i+1}-\chi_p^{i}) (L + b) 
	\label{eq:elem_eb_rep}
	\\
	\Rightarrow \qquad  & (\chi_{p}^{i+1} - \chi_{p}^{i}) L 
	=  b^{\o_p^i} - b^{\o_0} - (\chi_{p}^{i+1} - \chi_{p}^{i}) b_p^i.
	\tag{where $b_p^i \ll L$}
	\end{align*}
	Substituting the value of $(\chi_{p}^{i+1} - \chi_{p}^{i}) L$ in Eq. (\ref{exchange:eq3}), we get:
	\begin{align}
	b^{\ob} - b^{\o_s} & = \sum_{j=1}^T a_j \sum_{i=l_j}^{k_j} ( b^{\ot_p^i} - b^{\o_0} - (\chi_{p}^{i+1} - \chi_{p}^{i}) b_p^i)  + 
	\sum_{j=1}^T b_j \sum_{i=l_j}^{k_j} (\chi_{p}^{i+1} - \chi_{p}^{i}) 
	\nonumber
	\intertext{
		Since both $\o_s$, and $\o_0$ are valid orderings, we can write $b^{\o_s} = b^{\o_0} + \vec{d}$, where elements of $\vec{d}$ are much smaller than $L$. Therefore we get
	}
	b^{\ob} & = \sum_{j=1}^T \sum_{i=l_j}^{k_j} a_j b^{\ot_p^i} 
	+ \Big( 1 - \sum_{j=1}^T  \sum_{i=l_j}^{k_j} a_j \Big) b^{\o_0} 
	- \sum_{j=1}^T \sum_{i=l_j}^{k_j} a_j (\chi_{p}^{i+1} - \chi_{p}^{i}) b_p^i   \\
	& \quad + \sum_{j=1}^T \sum_{i=l_j}^{k_j} (\chi_{p}^{i+1} - \chi_{p}^{i})b_j 
	+ \vec{d} 
	\nonumber \\
	b^{\ob} & = \sum_{j=1}^T \sum_{i=l_j}^{k_j} a_j b^{\ot_p^i} + \Lambda, 
	\label{exchange:eq4ms}
	\intertext{
		where Equation (\ref{exchange:eq4ms}) has been derived summing the last 4 terms into a vector $\Lambda$. Note that all the elements of $\Lambda$ are $\ll L$. It is easy to see that the first term in the equation essentially is a linear combination of some elementary invalid extreme bases, allowing us to simplify: 
	}
	b^{\ob} & = \sum_{p\in \partition}\sum_{i=1}^{|p|-1} \alpha_p^i b^{\ot_p^i}+ \Lambda,
	\label{eq:inv_eb_rep}
	\end{align}
	where coefficients $\alpha_p^i$ corresponding to elementary extreme bases not present in Equation (\ref{exchange:eq4ms}) can be simply set to zero.
\end{proof}

\noindent Due to Lemma \ref{lemma:ebequal}, the above result is also true for representing the invalid extreme bases (and not only the canonical ones), with a different $\Lambda$. Lemma \ref{lemma:elem-eb-chi-rep} allows us to further simplify the result of Lemma \ref{lemma:elem-eb-lin-comb} to the following:

\begin{lemma}[Invalid Extreme Base Representation]
An invalid extreme base can be represented as $b^\o = \sum_{p \in \partition} \sum_{i=1}^{|p|-1} \alpha_p^i L (\chi_{p}^{i+1} - \chi_{p}^{i}) + \Lambda$, where $\chi_p^i$ is an indicator vector corresponding to element $p^i$, $0 < \alpha_p^i \ll L$, and $\Lambda$ is some vector whose all elements are $\ll L$.
\label{lemma:elem-eb-lin-comb-chi}
\end{lemma}

\begin{proof}
Using Equation (\ref{eq:inv_eb_rep}), we have:
\[		
b^{\ob} = \sum_{p\in P}\sum_{i=1}^{|p|-1} \alpha_p^i b^{\ot_p^i}+ \Lambda.
\]
Substituting representation of elementary extreme base from Equation (\ref{eq:elem_eb_rep}), we have:
\begin{align*}
	b^{\o} & = \sum_{p\in P} \sum_{i=1}^{m-1} \alpha_p^{i} \Big( b^{\o_0}+ (\chi_{p}^{i}-\chi_{p}^{i+1}) (L + b_p^i) \Big)+ \Lambda. 
	\\
	& = \sum_{p\in P} \sum_{i=1}^{m-1} \alpha_p^{i} b^{\o_0} 
	  + \sum_{p\in P} \sum_{i=1}^{m-1} \alpha_p^{i} L (\chi_{p}^{i+1}-\chi_{p}^{i}) 
	  + \sum_{p\in P} \sum_{i=1}^{m-1} (\chi_{p}^{i+1} - \chi_{p}^{i}) \alpha_p^i b_p^i + \Lambda. 
	\\
	& = \sum_{p\in P} \sum_{i=1}^{m-1} \alpha_p^{i} L (\chi_{p}^{i+1} - \chi_{p}^{i}) + \Lambda. 
\end{align*}
	Note that we have replaced $\Lambda$ with $\sum_{p\in P}\sum_{i=1}^{|p|-1} \alpha_p^{i} b^{\o_0} + \sum_{p\in P}\sum_{i=1}^{|p|-1} (\chi_{p}^{i+1}-\chi_{p}^{i}) \alpha_p^i b_p^i + \Lambda$.
\end{proof}

Recall from Eq. (\ref{eq:base_split_valid_invalid}): $x = x_v + x_i$, where $x_v = \sum_{b^{\o_j} \in R} \lambda_j b^{\o_j}$, and $x_i = \sum_{b^{\o_i} \in Q} \lambda_i b^{\o_i}$. Using Lemma \ref{lemma:elem-eb-lin-comb-chi} to replace the second term, and noting that $L \approx \infty \Rightarrow \lambda_i \approx 0, \text{and} \sum \lambda_j \approx 1$, one observes that the term $\sum_{b^{\o_i} \in Q} \lambda_i \Lambda_i$ in the expansion can be made smaller than the precision constant by increasing the value of $L$ ( $\lambda < |\VS|M/L$ by Lemma \ref{lamba-for-invalid-eb}) and can be dropped. As one of the final theoretical results of this paper, we can show the following: 
\begin{theorem}[Main Result]
\label{thm:inv-block-rep}
\be \sum_{\forall b^{\o_i} \in Q} \lambda_i b^{\o_i} = 
	\sum_{ p \in \partition} \sum_{k=1}^{|p|-1} 
	\beta_p^k L (\chi_p^{k+1} - \chi_p^{k}), \ee
where $\lambda_i \geq 0$, $\beta_p^k=\sum_{b_i \in Q} \alpha_p^k \lambda_i$.
\end{theorem}

\begin{proof}

Consider the expansion of the term  $x_i = \sum_{b^{\o_i} \in Q} \lambda_i b^{\o_i}$ in Eq.(\ref{eq:base_split_valid_invalid}). Using Theorem (\ref{lemma:elem-eb-lin-comb-chi}) we get:
\begin{align*}
	x_i & = \sum_{b^{\o_i} \in Q} \sum_{p\in \partition} \sum_{k=1}^{|p|-1} 
 		\lambda_i \alpha_p^k  L(\chi_{p}^{k} - \chi_{p}^{k+1})  
 		+ \sum_{b^{\o_i} \in Q} \lambda_i \Lambda_i. \\
	\intertext{
		Recall from Lemma (\ref{lamba-for-invalid-eb}) that, for all $b^{\o_i} \in Q$, the coefficient $\lambda_i$ can be made arbitrarily small. Therefore, we can drop the term $\sum_{b^{\o_i} \in Q} \lambda_i \Lambda_i$ and rewrite the above equation as:
	}
 	x_i & = \sum_{p\in P} \sum_{k=1}^{|p|-1} \sum_{b^{\o_i} \in Q}  
 				\lambda_i \alpha_p^k  L(\chi_{p}^{k+1} - \chi_{p}^{k}). 
	\intertext{
		Replacing by $\beta^k_p = \sum_{b^{\o_i} \in Q}  \lambda_i \alpha_p^k$, we get:
	}
 	x_i & = \sum_{p\in P}\sum_{k=1}^{|p|-1} \beta_p^k L(\chi_{p}^{k+1} - \chi_{p}^{k}).
\end{align*}

\end{proof}
\noindent Note that the above result incorporates \ue{all} the invalid extreme bases, not merely the ones involved in the representation of base vector $x$ in any iteration of MNP. Using the result in Eq. (\ref{eq:base_split_valid_invalid}), we get: 
$ 
\norm{x}^2 = \norm{\sum_{b^{\o_j} \in R} \lambda_j b^{\o_j} + \sum_{ p \in \partition} \sum_{k=1}^{m-1} \beta_p^k L (\chi_p^{k+1} - \chi_p^{k})}^2.
$

\section{Main Algorithm}
\label{sec:proposed-method}

In this section we give the algorithm for minimizing the norm of the base vector corresponding to a simple submodular system. Where the pseudo-Boolean function is generated by using extension defined in definition \ref{definef}. For minimizing the sum of submodular functions defined on lattices, the proposed algorithm can be used in the inner loop of the BCD strategy as suggested in \cite{sosmnp}.

\setlength{\columnsep}{10pt}%
\setlength{\intextsep}{2pt}%
\begin{wrapfigure}{r}{0.4\linewidth}
	\begin{center}
		\includegraphics[width=1.0\linewidth]{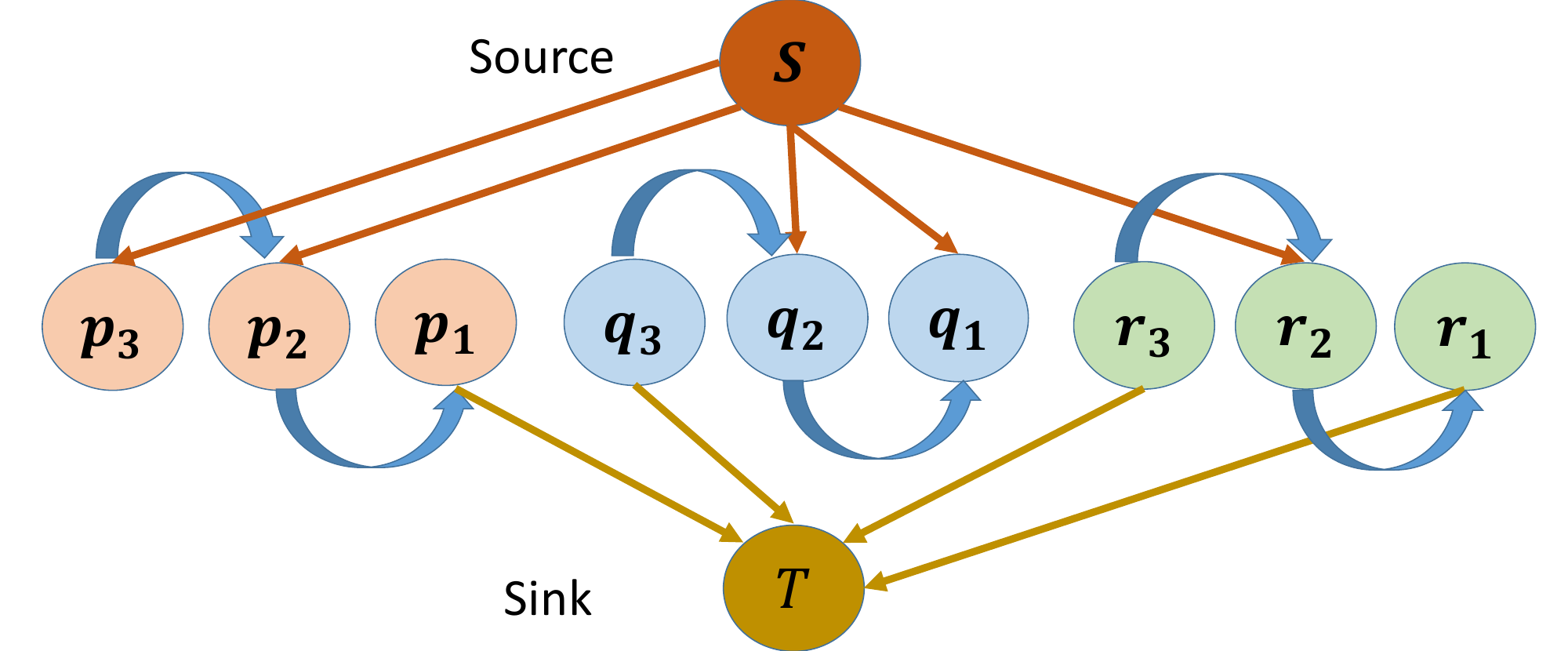}
	\end{center}
	\caption{Flow graph corresponding to the exchange operations for optimizing the block containing invalid extreme bases.}
	\label{Fig:FlowGraph}
\end{wrapfigure}

Theorem (\ref{thm:inv-block-rep}) opens up the possibility of minimizing $\norm{x}^2$ for a submodular function defined over a lattice using the BCD strategy. We will have two blocks. The first block, called the \emph{valid block}, is a convex combination of valid extreme bases $b^{\o_j}$, where standard MNP algorithm can be used to optimize the block. The other block, called the \emph{invalid block}, corresponds to the sum of the terms of type: $\beta_p^k L (\chi_p^{k+1} - \chi_p^{k})$, representing the invalid extreme bases. 
For minimizing the norm of the overall base vector using the invalid block, we hold the contribution from the valid block, $x_v$, constant \footnote{Recall that we start from a valid extreme base. Therefore, at initialization $x=x_v$}. Each vector $\beta_p^k L (\chi_p^{k+1} - \chi_p^{k})$ may be looked upon as capturing the $\beta_p^k$ increase/decrease due to the exchange operation between the two adjacent elements which define an elementary extreme base. This exchange operation can be viewed as flow of $\beta_p^k L$ from the element $p^{k}$ to $p^{k+1}$. We model the optimization problem for the invalid block using a flow graph whose nodes consists of $\{p^k \mid p \in \partition , 1 \leq k \leq |p|-1\} \cup \{s, t\}$. We add two type of edges:
\begin{itemize}[leftmargin=*]
	\item \textbf{Type 1:} If $x_v(p^k)$, corresponding to the valid block contribution, is $> 0$, then we add a directed edge from $s \ra p^k$, else we add the edge from $p^k \ra t$ with capacity $x_v(p^k)$. 
	\item \textbf{Type 2:} The directed edges $p^{k}$ to $p^{k+1}$, $1 \leq k \leq (|p|-1) $ with capacity $\card{\VS}M$ to ensure that the capacity is at least as large as $\beta_p^k L$: much larger than any permissible value of $x_v(p^k)$. Thus, any feasible flow augmentation in a path from from $s$ to $t$ can saturate only the first or the last edge in the augmenting path (i.e. the edge emanating from $s$ or the edge incident at $t$ in the path). 
\end{itemize}

Figure \ref{Fig:FlowGraph} is an example of a flow graph which has $3$ segments and each segment has $3$ elements. Since the starting state is $x_v$ the ``initial flow'' prior to pushing flow for flow maximization requires setting flow in a  type 1 edge incident at $p^k$ equal to the value of $x_v(p^k)$ and that in type 2 edges as $0$. This is because sum of flow on all edges incident at a node may be looked upon as the value of the corresponding element in the base vector \footnote{we refer the reader to  \cite{shanu2018inference} for details about the flow to base vector correspondence}. In effect initially there are non zero excesses on the non $s, t$ nodes in the flow graph defined as the sum of net in-flow on all edges incident at a node. The excess at node $p_k$ is denoted by $e(p_k)$. Max flow state can be looked upon as that resulting from repeatedly sending flow from a positive excess vertex to a negative excess vertex till that is no more possible. Values in the optimal base vector (optimal subject to the given $x_v$) at the end of this iteration will be the excesses at nodes when max flow state has been reached. 

\begin{algorithm}[t]
	\caption{Computing Min $\ell_2$ Norm from the Flow Output}
	\label{algo:phi}
	\begin{algorithmic}[1]
		\REQUIRE Vector $e$ the output of the MNP algorithm over valid block.
		\ENSURE The transformed vector $e$ with minimum $\ell_2$ norm.
		\FOR {$\forall p \in \partition $}
		\FOR {$i=1:|p|-1$} 
		\REPEAT
		\STATE Find maximum $k$, $i \le k \le |p|$,  such that \\ $e(p^i) > e(p^{i+1}) = e(p^{i+2}) \cdots = e(p^{k})$ or \\  $e(p^i) = e(p^{i+1}) = e(p^{i+2}) \cdots = e(p^{k-1}) > e(p^{k})$; \label{algo:step4}
		\STATE Set $e(p^i),e(p^{i-1}),\ldots,e(p^k)$ equal to $av_k$, \\
		where $av_k$ is the average of $e(p^i),e(p^{i-1}),\ldots,e(p^k)$; 
		\UNTIL {$e(p^{k}) \leq e(p^{k+1})$}
		\ENDFOR
		\ENDFOR
	\end{algorithmic}
	\label{algo:invalidblock}
\end{algorithm}

\subsection{Computing Min $\ell_2$ Norm By Flow}
\label{sec:invalidl2norm}
Since there is no edge between any two nodes corresponding to different segments max flow can be calculated independently for each segment. When max flow state is reached in the flow graph associated with a segment, a vertex which still has a negative excess will be to the left of vertices with positive excess (planar flow graph laid out as in Figure \ref{Fig:FlowGraph}) otherwise flow could be pushed from a positive excess vertex to a negative excess vertex. 

Note that the optimal base vector is not unique. Consider two adjacent vertices, $p^{k}$ and $p^{k+1}$, in the flow graph when the max flow state has been reached. If $e(p^{k})$ is larger than $e(p^{k+1})$ then increasing the flow in the edge from $p^{k}$ to $p^{k+1}$ by $\delta$ decreases $e(p^{k})$ by $\delta$ and increases $e(p^{k+1})$ by $\delta$. The result of this ``exchange operation'' is to create another optimal base vector but with a smaller $\ell_2$ norm. 

An optimal base vector with minimum $\ell_2$ norm will correspond to the max flow state in the flow graph in which $e(p^{k}) \leq e(p^{k+1})$  for all adjacent pairs of type 2 vertices. If this is not so then there would exist at least a pair $e(p^{k})$ and $e(p^{k+1})$ such that $e(p^{k}) > e(p^{k+1})$. Doing an exchange operation between $p^{k}$ and $p^{k+1}$ involving setting $e(p^{k+1})$ and $e(p^k)$ to the average of the old values will create a new optimal base vector with lower value of the $\ell_2$ norm. Algorithm \ref{algo:phi} gives an efficient procedure to transform the optimal base vector outputted by the max flow algorithm to one with minimum $\ell_2$ norm. Note that the proposed algorithm simply updates the base vector in one pass without any explicit flow pushing. In contrast, the corresponding algorithm for general flow graphs given in \cite{shanu2018inference} requires $O(n\log{n})$ additional max flow iterations over an $n$ vertex flow graph. 

\subsection{Overall Algorithm}
\label{sec:mainalgo}
Our algorithm builds upon \cite{fujishige2006minimum} MNP algorithm.
Recall that in a standard MNP algorithm iteration, given the current base vector  $x$, an extreme base $q$, that minimizes $x^{\intercal}q$ is added to the current set. Hence, steps to convergence of MNP is bounded by the number of extreme bases that may be added. In our case we will show that when we  start with a valid extreme base, the extreme base generated in the valid block after using the latest contribution from the invalid block, will come out to be a valid extreme base. 

\begin{lemma}
Algorithm \ref{algo:phi} returns a vector $x$ such that extreme base $q$ given as $q = \argmin_{q \in B(f)} x^T q$ is valid.
\label{lemma:invalidtovalid}
\end{lemma}
\begin{proof}
It is easy to show that when Algorithm \ref{algo:phi} terminates, for any pair of indices $i, j$ corresponding to any segment $p \in \partition$ if $i < j$ then $e(p^i) \leq e(p^j)$. Note that by construction the excess vector $x$ is the base vector. This implies that the order $\o$ of the indices obtained by sorting the elements of $x$ will satisfy  $p^i ~ \o ~ p^j, \forall i<j$, and $ \forall p \in \partition$. This is the condition that has to be satisfied for an ordering to be valid (Cf. Def. \ref{def:validinvalid}). Recall that in the MNP algorithm the extreme base is found by computing the ordering of sorted elements of $x$. Hence, the extreme base $q = \argmin_{q \in B(f)} x^T q$ will be a valid one.
\end{proof}

\begin{algorithm}[t]
\caption{Main Algorithm}
\label{algo:main}
\begin{algorithmic}[1]
\REQUIRE Submodular function: $f$
\REQUIRE Modular vector $a$
\REQUIRE Partial order $\oe_{\DS}$
\ENSURE Base vector $x^{*} \in B(f + a)$ minimizing $\norm{x}^2$

\LBLANKCOMMENT
\WHILE{(TRUE)}
\STATE Put $f = f + a$
\STATE S = \{\}
\STATE Find extreme base $\hat{q} := \argmin\limits_{q \in B_{f}} \dotprod{x}{q}$ using Edmond's algorithm.
  \IF{ Extreme base $\hat{q}$ is invalid according to the partial order $\oe_{\DS}$. (Definition \ref{def:validinvalid})}
    \label{algo2:step5}
    \STATE $x$ = Compute Invalid Contribution($x$) \ref{algo:phi};
    \label{algo2:step6}
    \STATE \textbf{continue};
  \ENDIF
  \IF{ ($\norm{x}^2 \leq \dotprod{x}{\hat{q}} + \epsilon$) } \label{epsilon-step}
    \STATE break;
  \ENDIF
  \STATE $S := S \cup \hat{q}$;
  \STATE Find $x$ in affine hull of $S$;
  \STATE If $x$ is not in convex hull $S$, translate to nearest point in convex hull and update $S$;
\ENDWHILE
\end{algorithmic}
\end{algorithm}

This implies that the number of iterations involving invalid blocks can not exceed the number of valid extreme bases added as in the standard MNP algorithm. This ensures convergence of the algorithm. The formal convergence proof for the algorithm is given in the later Section.

The correctness of our algorithm follows from the fact that the optimization for valid blocks proceeds in the standard way, and results in a new extreme base given the current base vector. The correctness of the optimization step of the invalid block,  which finds a minimum norm base vector given a valid block, has already been explained in the previous section.

An implementation of our main algorithm is present in Algorithm \ref{algo:main}. Note that we optimize function $f+a$, where $a$ is a modular vector. It gives algorithm flexibility to be plugged directly into our sum of submodular function minimization Algorithm. Input to algorithm is a submodular function and a modular vector. We can compute the extreme base for translated function using lemma \ref{lemma:extremetranslation}. All other steps are very similar to Min Norm Point Algorithm \cite{fujishige06} except step \ref{algo2:step5} and step \ref{algo2:step6}. In these steps we check if new extreme base computed is invalid or not. If it comes out to be invalid then we call algorithm \ref{algo:invalidblock} to make vector $x$ consistent. So that it outputs vector $x$ which will produce a valid extreme base. If the new extreme point comes out to be valid then later steps follow like Min Norm Point Algorithm.

\subsection{Sum of Submodular Minimization}

In this section we extend our algorithm for minimizing sum of submodular functions defined over different lattices. The proposed algorithm is quite similar to the algorithm in \cite{sosmnp} in its over all structure. Just like \cite{sosmnp}, we also create blocks corresponding to each submodular function, and optimize each block independently (taking the contribution of other blocks as suggested in \cite{sosmnp}) in an overall block coordinate descent strategy. The only difference between SoSMNP and our algorithm is the way we optimize one block. While SoSMNP uses standard MNP, we optimize using a special technique, as outlined in previous section, with (sub)blocks of valid and invalid extreme bases, within each block/clique. Hence, the convergence and correctness of overall algorithm follows from block coordinate descent similar to \cite{sosmnp}. What we need to show is that for a single block, the algorithmic strategy of alternating between valid and invalid blocks converges to the optimal for that block.

We give the proposed sum of submodular minimization method in Algorithm \ref{algo:mlhybrid}. The algorithm takes submodular $f_\cs$'s and computes minimum $\ell_2$ norm of $x \in B(f)$ s.t. $f = \sum_{\cs \in \CS} f_\cs$. The overall algorithm solves valid block with the SoS-MNP algorithm given in \cite{sosmnp} and uses Algorithm \ref{algo:phi} to solve invalid block. Let $x_\cs$ be the restriction of $x$ over clique $\cs$, the norm $\norm{x}^2$ is optimized by computing minimum norm $\norm{x_\cs}^2$ over each clique cyclically. Algorithm \ref{algo:mlhybrid-over-a-clique} minimizes $\norm{x_\cs}^2$ in a very similar way as MNP Algorithm \cite{sosmnp} described in Background section. The only difference lies in handling the invalid extreme base at step \ref{algo3:step5} of Algorithm \ref{algo:mlhybrid-over-a-clique}.

\begin{algorithm}[t]
\caption{Algorithm for minimizing a sum of  submodular functions defined over lattices}
\label{algo:mlhybrid}
\begin{algorithmic}[1]
\REQUIRE $\{f_\cs\}$ such that $f=\sum f_\cs$.
\ENSURE $x = \argmin \norm{x}^2$ subject to $x \in B(f)$.
\LCOMMENT{Initialize}
\FORALL{($\cs \in \CS$)}
  \STATE Compute partial order $\oe_{\DS^\cs}$
  \STATE $q_\cs$ $\leftarrow$ Take any extreme base of $f_\cs$;
  \STATE $S_\cs := \{ q_\cs \}$;
  \STATE $y_\cs := q_\cs$;
\ENDFOR
\STATE $x := \sum_\cs y_\cs$;
\LCOMMENT{Perform Block Coordinate Descent with blocks specified by blocks}
\WHILE{($\norm{x}$ decreases by more than $\delta$)}
  \FORALL{($\cs \in \CS$)}
    \STATE $a_\cs = x_\cs - y _\cs$
    \STATE $x_\cs$ = Main Algorithm($f_\cs$,$a_\cs$, $\oe_{\DS^\cs}$);
  \ENDFOR
\ENDWHILE
\end{algorithmic}
\end{algorithm}

\begin{algorithm}[t]
\caption{MLHybridOverAClique}
\label{algo:mlhybrid-over-a-clique}
\begin{algorithmic}[1]
\REQUIRE Clique function: $f_\cs$
\REQUIRE Set of valid extreme bases selected in last iteration: $S_\cs$
\REQUIRE Restriction of current solution vector $x$ on $\cs$: $x_\cs$
\REQUIRE Current clique vector: $y_\cs$
\ENSURE Clique vector $y_\cs^{*} \in B(f_c)$ minimizing $\norm{x_\cs}^2$
\ENSURE Updated set $S_\cs^{*}$ of valid extreme bases
\LBLANKCOMMENT
\WHILE{(TRUE)}
  \STATE Find new translation $a_\cs := x_\cs - y_\cs$;
  \STATE Find extreme base $\hat{q}_\cs := \argmin\limits_{q_\cs \in B_{f_\cs}} \dotprod{x_\cs}{q_\cs}$ using Edmond's algorithm.
  \IF{ Extreme base $\hat{q}_\cs$ is invalid according to Definition \ref{def:validinvalid}}
    \label{algo3:step5}
    \STATE ComputeInvalidContribution($x_\cs$);
    \STATE \textbf{continue};
  \ENDIF
  \STATE Find translated extreme base $\hat{p}_\cs = \hat{q}_\cs + a_\cs$;
  \IF{ ($\norm{x_\cs}^2 \leq \dotprod{x_\cs}{\hat{p}} + \epsilon$) } \label{epsilon-step}
    \STATE break;
  \ENDIF
  \STATE $S_\cs := S_\cs \cup \hat{q}_\cs$;
  \STATE $P_\cs=\{\hat{q_\cs} + a_\cs | q_\cs \in S_\cs\}$;
  \STATE Find $x_\cs$ in affine hull of $P_\cs$;
  \STATE If $x_\cs$ is not in convex hull $P_\cs$, translate to nearest point in convex hull and update $S_\cs$;
\ENDWHILE
\end{algorithmic}
\end{algorithm}

\subsection{Convergence of SoSMNP \cite{sosmnp}}
\label{sosConvproof}

Our focus initially is to show the convergence to the optimal solution by the MNP algorithm running in the block co-ordinate descent mode as in \cite{sosmnp}. The problem formally is to minimize the function $f(S) = \sum_{\cs \in \CS} f_\cs(S \cap \cs) \quad S \subseteq \VS$, where $f_\cs: 2^{|\cs|} \rightarrow \mathcal{R}$ is a  submodular function. It has been shown in \cite{sosmnp} that $f$ can be minimized by finding a point $x \in B(f)$ with the minimum \ml2-norm $\norm{x}^2$. We write $x$ as the sum $x= \sum_{\cs \in \CS} y_\cs$ where $y_\cs \in B(f_\cs)$. 

We assume that a block corresponds to a  clique in $\CS$. Let $x_\cs$ be the restriction of $x$ to the elements in $\cs \in \CS$,  and let $x_{\not \cs}$ be the restriction of $x$ on the remaining elements. We can write $\norm{x}^2 = \norm{x_\cs}^2 + \norm{x_{\not \cs}}^2$.  The block co-ordinate descent algorithm in \cite{sosmnp} minimizes $\norm{x_\cs}^2$ using MNP over all the cliques $\cs \in \CS$ cyclically. This norm minimization step can be viewed as MNP minimizing $f'_\cs(S) = f_\cs(S) + a_\cs(S), \forall S \subseteq \cs$ where $a_\cs = x_c-y_c$, is a denoting the contribution of the other cliques which remains constant while running MNP over this clique/block. Note that $a_\cs(S) = \sum_{e \in S} a_\cs(e)$, and we can equivalently treat $a_\cs$ as a modular function as well. Let $f'_\cs(S) = f_\cs(S) + a_\cs(S), \forall S \subseteq \cs$. Note that the $f'$ as shown above is a sum of submodular ($f$), and a modular function ($a_\cs$). Therefore, $f'$ is submodular. It is easy to show the following result:
\begin{lemma}
Let $q_\cs$ be a extreme base vector in  $B(f_\cs)$ corresponding to an ordering $\o_\cs$. Then the vector $q_\cs + a_\cs$ is an extreme base of $B(f'_\cs)$ corresponding to the same ordering $\o_\cs$.  
\label{lemma:extremetranslation}
\end{lemma}
\begin{proof}
We can calculate the elements in extreme base vector ($q'_\cs \in B(f')$)  corresponding to ordering $\o_\cs$ by Edmond's Greedy Algorithm,
\begin{align*}
	q'(e) &= f'(S_e \cup e) - f'(S_e),    
	\tag{$S_e$ is the set of elements before $e\in \cs$ in $\o_\cs$.} \\
	& = f(S_e \cup e) + a_\cs(S_e \cup e)  - (f(S_e) + a_\cs(S_e)), \\
	& = f(S_e \cup e) + a_\cs(S_e) + a_\cs(e) - (f(S_e) + a_\cs(S_e)),  
	\tag{$a_\cs$ can be seen as a modular function.} \\
	& = f(S_e \cup e) - f(S_e) + a_\cs(e), \\
	& = q_\cs(e) + a_\cs(e).      \tag{By Edmond's Greedy Algorithm.} 
\end{align*}
Hence, $q'_\cs = q_\cs + a_\cs$.
\end{proof}
\noindent It is easy to see that:
\begin{align*}
x_\cs & = y_\cs + a_\cs = \sum_i \lambda_i q_\cs + a_\cs 
\tag{where $\sum_i \lambda_i=1$, and $\lambda_i$ > 0} \\
& = \sum_i \lambda_i q_\cs + \sum_i \lambda_i a_\cs 
\tag{Since $\sum_i \lambda_i=1$} \\
& = \sum_i \lambda_i (q_\cs + a_\cs) =  \sum_i \lambda_i q'_\cs  
\end{align*}
\noindent Hence, $x_\cs$ is a base vector of $f'$. Therefore, minimizing the minimum norm over a block, the way SoSMNP does it, can be seen as minimizing the norm of $x_\cs$: the restriction of $x$ over the elements of clique $\cs$ (and not $y_\cs$).
Let us suppose, we have reached a situation where the SoSMNP performs minimization over all blocks (cliques), and no change was observed in any of the blocks. The following lemma establishes the relationship between the extreme base of $f_\cs$, and the one corresponding to $f$.
\begin{lemma}
Let $q_\cs = \argmin_{q \in B(f_\cs)} x_\cs^T q, \; \forall \cs \in \CS$. Then $b = \sum_{\cs \in \CS} q_\cs$ also satisfies $b = \argmin_{b \in B(f)} x^Tb$.
\label{lemma:ebrelation}
\end{lemma}
\begin{proof}
In the SoSMNP algorithm, the extreme base $q_\cs$ is generated using Edmond's Greedy Algorithm \cite{schrijver00} on the order $\o_\cs$ of the indices obtained by sorting the elements of $x_\cs$ in the increasing order. We represent the extreme base so obtained by $q_\cs^{\o_\cs}$. The SoSMNP algorithm for a block terminates when $x_\cs^T x_\cs = x_\cs^T (q_\cs^{\o_\cs} + a_\cs)$

Consider the termination situation of SoSMNP for the overall problem (comprising of all the cliques). In such a case the algorithms tries to minimize for all the blocks/cliques and no change is found on any of the cliques. Therefore, termination condition of each block is met, and $q_\cs^{\o_\cs} = \argmin_{q \in B(f_\cs)} x_\cs^T q$. 

Let $\o_f$ be the ordering of elements of $x$ in the increasing order. It is easy to see that the ordering over $x$ and $x_\cs$ will be consistent with each other, in the sense that $x(e_1) ~ \o_f ~ x(e_2) \Rightarrow x_\cs(e_1) ~ \o_\cs ~ x_\cs(e_2)$. 

Let us create an extreme base of $f$, corresponding to the ordering $\o_f$, and denote as $b^{\o_f}$. Since $\o_f$ denotes the ordering over elements of $x$, therefore, from Edmond's algorithm, we have: $b^{\o_f} = \argmin_{b \in B(f)} x^Tb$. Further, we also have:
\begin{align*}
	b^{\o_f}(e) &= f(S_e \cup e)- f(S_e),     
	\tag{As per Edmond's algorithm. $S_e$ is the set of elements before $e$ in $\o_f$} \\
	& = \sum_{\cs \in \CS} f_{\cs}(S_e \cup e \cap \cs)- f(S_e \cap \cs),  
	\tag{Since $f(S) = \sum_{\cs \in \CS} f_\cs(S \cap \cs)$ } \\
	& = \sum_{\cs \in \CS} q_\cs^{\o_\cs}(e \cap \cs)    
	\tag{Since $\o_c$ is the restriction of $\o$}.
	\intertext{
		Since above holds for all the elements $e \in \VS$, therefore:
	}
	b^{\o_f} & = \sum_{\cs \in \CS} q_\cs^{\o_\cs}. 
\end{align*}
Hence, we have proved both the properties of $b^{\o_f}$
\end{proof}
We can now give the convergence proof of the SoSMNP with the following lemma:
\begin{lemma}
	If in a complete cycle of SoSMNP over all the cliques, we can not improve the norm $x_\cs$ for any $\cs$, then we have $x \in B(f)$ such that $\norm{x}^2 = x^T x = x^T b$, where $b = \argmin_{b \in B(f)} x^Tb$. 
	\label{lemma:convergence}
\end{lemma}
\begin{proof}
Recall that for a clique $\cs$, SoSMNP can be seen as minimizing the norm of $x_\cs$ which is a base vector of $f'_\cs = f_\cs + a_\cs$. Further $q'_\cs = q_\cs + a_\cs$ is an extreme base of $f'$. Therefore, from the termination of basic MNP algorithm, the following must hold:
\begin{align*}
    x_\cs^T x_\cs & = x_\cs^T (q_\cs + a_\cs). 
    \tag{$q_\cs = \argmin_{q\in B(f_\cs)} x_\cs^T q$} \\
    \intertext{
    	Summing over all the cliques we get
    }
    \sum_{\cs \in \CS} x_\cs^T x_\cs & = \sum_{\cs \in \CS} x_\cs^T (q_\cs + a_\cs), \\
    \sum_{\cs \in \CS} x_\cs^T (y_\cs + a_\cs) & = \sum_{\cs \in \CS} x_\cs^T (q_\cs + a_\cs), \\
    \sum_{\cs \in \CS} x_\cs^T y_\cs & = \sum_{\cs \in \CS} x_\cs^T q_\cs. 
    \tag{$\sum_{\cs \in \CS} x_\cs^Ta_\cs$ cancels out} \\
    \intertext{
    	Since vector $y_\cs$ and $q_\cs$ have non-zero values only for elements in $\cs$. Therefore we can write $x_\cs^T y_\cs= x^T y_\cs$ and $x_\cs^T q_\cs= x^T q_\cs$. Substituting the values, we get:
    }
    \sum_{\cs \in \CS} x^T y_\cs & = \sum_{\cs \in \CS} x^T q_\cs, \\
    x^T\sum_{\cs \in \CS}y_\cs & = x^T\sum_{\cs \in \CS} q_\cs \\
    x^T x & = x^T b  \tag{where $b= \argmin_{b \in B(f)} x^Tb$, by Lemma \ref{lemma:ebrelation}}
\end{align*}
The equation above is the termination condition of basic MNP when run over the overall function $f$ \cite{chakrabarty2014provable}. Therefore, the lemma essentially proves that the basic MNP terminating with optimal solutions for all cliques/blocks implies that the the base vector obtained by summing up the base vectors of all the cliques/blocks is the optimal solution for the overall objective function.
\end{proof}

When MNP algorithm is run in the block co-ordinate descent mode it is easy to show that any decrease in the $\norm{x_c}^2$ of a clique decreases the over all $\norm{x}^2$ by the same amount because $\norm{x_{\not \cs}}$ is untouched when optimizing for $\cs$. Since at each cycle there is at least one clique for which $\norm{x_c}^2$ decreases, we can say that $\norm{x}^2$ decreases monotonically at each cycle. Note that Theorem 4 of \cite{chakrabarty2014provable} gives us a lower bound on the improvement in every MNP iteration. It follows that MNP algorithm running in block co-ordinate descent mode will have a provable rate of convergence. 

For the sake of completeness we will also like to point out that the optimal solution obtained when MNP is run globally also corresponds to the individual blocks having reached their local optima. 

\section{Convergence of ML-hybrid Algorithm}

Note that in SoSMNP each block is optimized using the MNP algorithm. In MLhybrid, on the other hand, each block is further subdivided. One corresponds to the set of valid extreme bases (the valid block) and the other to the set of invalid extreme bases (the invalid block) whose convex combination  defines the base vector $x_\cs$. MNP is run on the valid block. If at any iteration MNP \cite{fujishige2006minimum} inserts an invalid extreme base, the  flow based Algorithm \ref{algo:invalidblock} is run on the invalid block. We show below that when MNP is run on the valid block now (that is just after a run of the flow based algorithm on the invalid block) the extreme base generated will be valid.

Lemma \ref{lemma:invalidtovalid} implies that an iteration on the invalid block will be followed by the MNP algorithm making progress in the form of generation of a valid extreme base. Also note that the $\ell_2$ norm decreases when the flow based algorithm is run on the invalid block. Therefore, termination and convergence of the MLhybrid algorithm running on a clique/block follows along the same lines as that for the standard MNP algorithm \cite{chakrabarty2014provable}. 

Now we show that termination over a clique/block results in $x_\cs$ using which minimizer obtained comes on a valid state.

Note that generation of an invalid extreme base can always be followed by generation of a valid extreme base (by running the flow based algorithm on the invalid block). Therefore, at termination it is guaranteed that the order $\o_\cs$ of the indices obtained by sorting the elements of $x_\cs$ is valid. That is the optimal solution corresponds to a valid primal state. Hence, it follows, using Lemma \ref{lemma:convergence}, that the MLHybrid algorithm run in the block coordinate descent manner converges to the optimal.

	\section{Experiments}

The proposed framework generalizes the inference algorithm of Shanu et al.~\cite{shanu2020inference}, which was developed for multi-label MRF-MAP inference with cliques of size up to $100$. In that prior work, the algorithm was applied to two computer vision problems: pixel-wise object segmentation on noisy Pascal VOC images~\cite{pascal-voc-2012} and stereo correspondence on $200 \times 200$ Middlebury images~\cite{scharstein2002taxonomy}. Cliques were generated with SLIC~\cite{achanta2012slic} using both decomposable (sum of absolute label differences) and non-decomposable (concave-of-cardinality~\cite{zhang2015higher}) submodular potentials over small ($60$--$80$ pixels) and large ($300$--$400$ pixels) overlapping cliques. For segmentation, pixel likelihoods were obtained from Deeplabv3+~\cite{deeplabv3plus2018}; the hybrid inference improved mean IoU from $0.544$ to $0.566$ and $0.579$ for small and large cliques, respectively, with per-image IoU near $0.9$ when hyperparameters were tuned. For stereo matching, the method was compared against MPI, TRWS, MPLP, and $\alpha$-expansion~\cite{koller2009probabilistic,kolmogorov2006convergent,globerson2008fixing,boykov2001fast,gould2012darwin} on instances where competing methods could not handle pairwise potentials within cliques of size $50$ or larger. The algorithm also converged substantially faster than SOS-MNP~\cite{sosmnp} on the same stereo instances. Full experimental details, including visual comparisons, runtime analyses, and convergence plots, are provided in~\cite{shanu2020inference}.

We next assess the scalability of the proposed lattice-based algorithm as the number of labels and the number of pixels in the $K$-submodular function increase. Figure~\ref{fig:compsubgrad} compares our method with the subgradient method~\cite{boyd2003subgradient} under non-decomposable potentials. For the subgradient baseline, the submodular function is extended beyond the lattice domain using Schrijver's algorithm~\cite{schrijver2003combinatorial}. The left panel reports running time as a function of the number of pixels while holding the label count fixed at $4$; the right panel reports running time as a function of the number of labels while holding the number of pixels fixed at $10$. Both panels use a logarithmic vertical axis, which makes clear that the subgradient method is several orders of magnitude slower than our lattice-based approach.

We next compare our method with the minimum-norm-point (MNP) algorithm~\cite{fujishige06} on larger synthetic instances representative of practical problem sizes. Figure~\ref{fig:complml} summarizes the running time of both methods as problem size increases. The left panel varies the number of pixels while holding the label count fixed at $30$; the right panel varies the number of labels while holding the number of pixels fixed at $50$. As the plots show, MNP running time grows sharply with problem size, whereas the lattice-based method remains substantially faster over the full range tested.

\begin{figure}[t]
	\centering
	\includegraphics[width=0.95\textwidth]{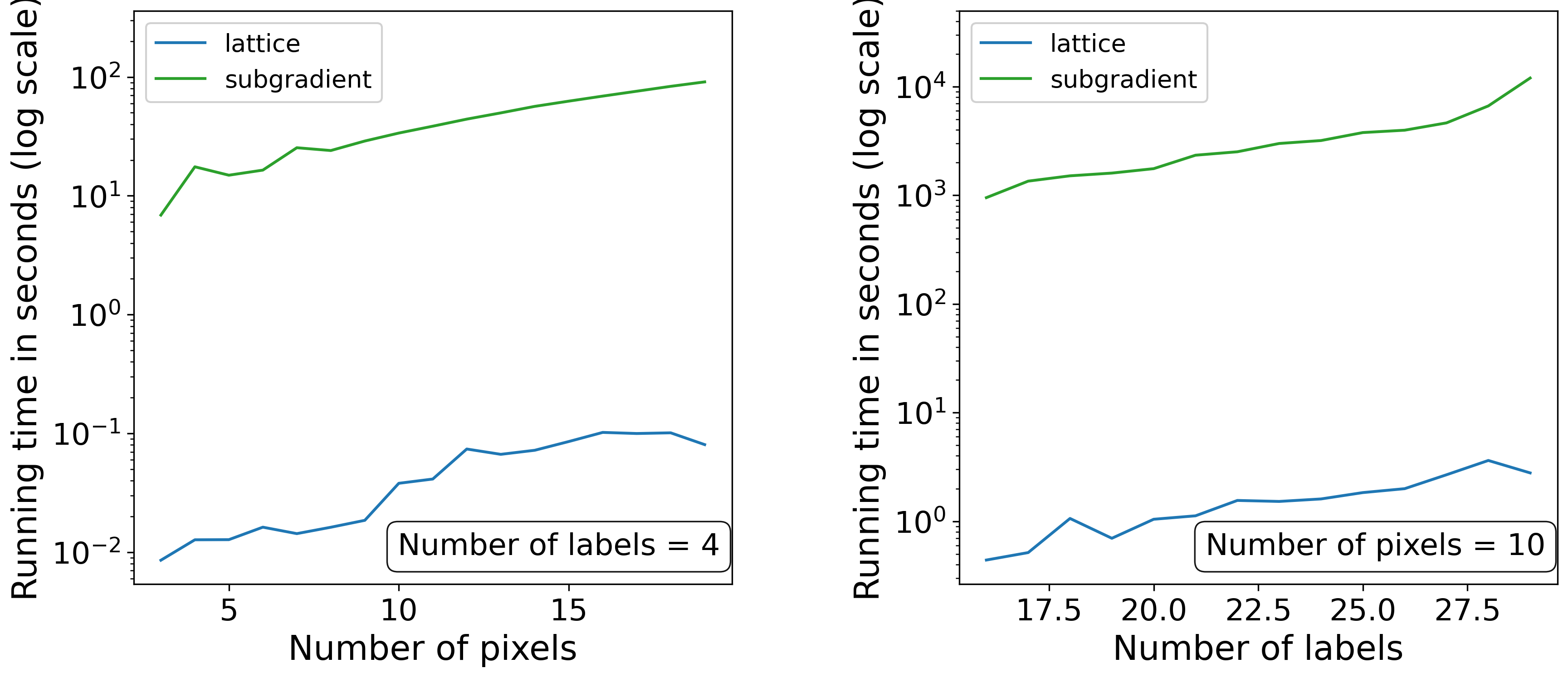}
	\caption{Scalability comparison with the subgradient method under non-decomposable potentials. Left: running time versus the number of pixels with the label count fixed at $4$. Right: running time versus the number of labels with the number of pixels fixed at $10$. Both panels use a logarithmic vertical axis.}
	\label{fig:compsubgrad}
\end{figure}

\begin{figure}[t]
	\centering
	\includegraphics[width=0.95\textwidth]{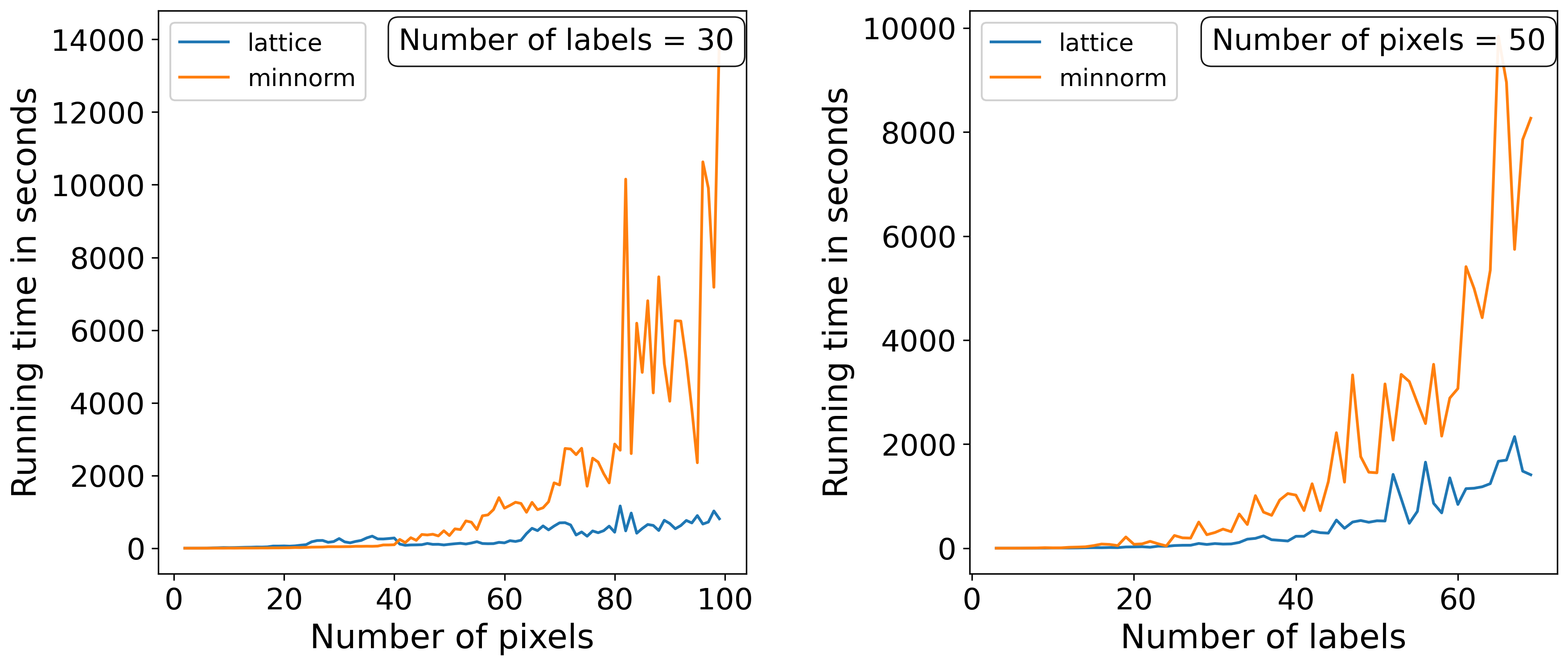}
	\caption{Scalability comparison with the minimum-norm-point (MNP) method on larger synthetic instances. Left: running time versus the number of pixels with the label count fixed at $30$. Right: running time versus the number of labels with the number of pixels fixed at $50$.}
	\label{fig:complml}
\end{figure}

	\section{Conclusions}

We have presented a generic framework for submodular function minimization defined on distributive lattices. Unlike the prevailing approach of extrapolating the problem to an equivalent Boolean lattice, our method operates entirely within the distributive lattice itself. This avoids the exponential blow-up in domain size that makes traditional Schrijver-style transformations impractical for real applications. Because the framework is lattice-native, established submodular minimization algorithms developed for Boolean lattices can be adapted and reused in a principled way, making the overall approach both modular and broadly applicable across distributive lattice structures.

Our experiments confirm that this design translates into substantial practical gains. Compared with traditional methods that first extend the submodular function beyond the lattice domain, our approach is several orders of magnitude faster in running time. Against the minimum-norm-point algorithm on larger instances, it remains substantially more efficient as the number of labels and elements grows. Across these comparisons, the proposed method is not only faster but also more stable: its performance degrades far less sharply with problem size, and its runtime behavior is more predictable than that of extrapolation-based baselines.

Taken together, the framework offers an efficient, stable, and generic solution for submodular minimization on distributive lattices. By eliminating the need to work in an artificially enlarged Boolean domain, it provides a practical route to scaling lattice-defined optimization problems in computer vision, machine learning, and related areas where distributive structure is inherent to the problem formulation.

\bibliographystyle{style-files/splncs04}
\bibliography{sections/egbib_ml}

\end{document}